\documentclass[a4paper,12pt]{article}
\usepackage{graphicx, amsfonts, amsthm, amsxtra, amssymb, verbatim, makeidx}
\usepackage{subeqnarray, relsize}
\usepackage[mathscr]{euscript}
\usepackage[english]{babel}
\usepackage[fixlanguage]{babelbib}
\usepackage[utf8]{inputenc}
\usepackage[english]{babel}

\usepackage{wrapfig}
\usepackage{amssymb, amsmath, amsthm}
\usepackage{amssymb, amsmath, amsthm}
\usepackage{graphicx}
\usepackage{color}
\usepackage{amssymb}
\usepackage{url}
\usepackage{pdfpages}
\usepackage{fancyhdr}
\usepackage{subfig}
\usepackage{titlesec}
\usepackage{enumerate}
\usepackage{comment}
\usepackage{bigints}
\usepackage{diagbox}
\usepackage{cite}
\usepackage{float}
\usepackage{algorithm}
\usepackage{algpseudocode}
\usepackage{array}
\usepackage{multirow}
\usepackage{MnSymbol} 
\usepackage{amsmath}

\usepackage[fixlanguage]{babelbib}
\usepackage[unicode,
            psdextra,
            colorlinks=true,
            linkcolor=blue,
            citecolor=green
            ]{hyperref}
\usepackage[nameinlink,capitalize,noabbrev]{cleveref}

\makeindex
\makeglossary

\newtheorem{thm}{Theorem}[section]
\newtheorem{prop}{Proposition}[section]

\newtheorem{rem}{Remark}[section]

\newtheorem{lem}{Lemma}[section]

\newtheorem{pro}{Property}
\newtheorem{examp}{Example}
\newtheorem{definition}{Definition}

\renewenvironment{proof}{{\bfseries \noindent Proof} }{ \qed \\}

\begin{document}

%%%%%%%%%%%%%%%%%%%%%%%%%%%%%%%%%%%%%%%%%%%%%%%%%%%%%
%----------------General Math Notations-------------------------------------
\def\M{\mathcal{M}}
\def\R{\mathbb{R}}                   %real numbers
\def\Z{\mathbb{Z}}                   %Integer  numbers
\def\Q{\mathbb{Q}}                   %Rational  numbers
\def\C{\mathbb{C}}                   %Complex numbers
\def\N{\mathbb{N}}                   %natural numbers
\def\uhp{{\mathbb H}}                %upper half plane
\def\A{\mathbb{A}}
\def\vol{\textbf{vol}}

\newcommand{\cuater}{\mathbf{\mathbb{H}}}
\newcommand{\complex}{\mathbf{\mathbb{C}}}
\newcommand{\reales}{\mathbf{\mathbb{R}}}
\newcommand{\racio}{\mathbf{\mathbb{Q}}}
\newcommand{\ente}{\mathbf{\mathbb{Z}}}
\newcommand{\natu}{\mathbf{\mathbb{N}}}
\newcommand{\llave}[1]{\left\{ #1\right\}}
\newcommand{\norma}[1]{\left \| #1 \right \|}
\newcommand{\vabs}[1]{\left| #1\right|}
\newcommand{\Sum}[2]{\sum\limits_{#1}^{#2}}
\newcommand{\Lim}[1]{\lim\limits_{#1}}
\newcommand{\Liminf}[1]{\liminf\limits_{#1}}
\newcommand{\Limsup}[1]{\limsup\limits_{#1}}
\newcommand{\Inf}[1]{\inf\limits_{#1}}
\newcommand{\Sup}[1]{\sup\limits_{#1}}
\newcommand{\corch}[1]{\left[ #1\right]}
\newcommand{\paren}[1]{\left( #1\right)}
\newcommand{\QED}{\hfill \ensuremath{\Box}}
\renewcommand{\baselinestretch}{1}
\providecommand{\keywords}[1]{\textbf{Keywords: } #1}
\providecommand{\classifications}[1]{\textbf{Mathematics Subject Classification:} #1}
% \bibliographystyle{siam}%{unsrt}

%%%%%%%%%%%%%%%%%%%%%%%%%%%%%%%%%%%%%%%%%%%%%%%%%%
%%%%%%%%%%%%%%%%%%%%%%%%%%%%%%%%%%%%%%%%%%%%%%%%%%

\begin{center}
{\LARGE\bf  Hodge Laplacians and Hodge Diffusion Maps
}

\vspace{.25in} {\large {{\sc  Alvaro Almeida Gomez}}\footnote{
Universidad de Chile, Centro de Modelamiento Matemático (CNRS IRL2807), Beaucheff 851, Santiago, Chile,
{\tt alvaroalmeidagomez182@gmail.com}}
\large{{\sc and Jorge Duque Franco}}\footnote{ Universidad de Chile, Departamento de Matemáticas, Campus Juan Gómez Millas, Las Palmeras 3425, Santiago, Chile,
{\tt jorge.duque@algebraicgeometry.cl}}}
\end{center}

% \tableofcontents

\begin{abstract}
We introduce Hodge Diffusion Maps, a novel manifold learning algorithm designed to analyze and extract topological information from high-dimensional data-sets. This method approximates the exterior derivative acting on differential forms, thereby providing an approximation of the Hodge Laplacian operator. Hodge Diffusion Maps extend existing non-linear dimensionality reduction techniques, including vector diffusion maps, as well as the theories behind diffusion maps and Laplacian Eigenmaps. Our approach captures higher-order topological features of the data-set by projecting it into lower-dimensional Euclidean spaces using the Hodge Laplacian. We develop a theoretical framework to estimate the approximation error of the exterior derivative, based on sample points distributed over a real manifold. Numerical experiments support and validate the proposed methodology. 

\end{abstract}

\\

\classifications {Primary: 68P05, 68T10, 68T45, 68W25; Secondary: 20G10, 35J05, 58J35, 58A14 .}

\section{Introduction}
\label{intro}

Dimensionality reduction is an essential technique for analyzing complex, high-dimensional datasets. It helps uncover important patterns and structures while overcoming the challenges of the \textbf{curse of dimensionality}. One popular nonlinear dimensionality reduction method is Diffusion Maps (DM) \cite{coifman2006diffusion,lafon2004diffusion}, a graph-based kernel method. Diffusion Maps captures the intrinsic geometry of data through a nonlinear embedding by using diffusion processes on a graph. This approach measures local connectivity between data points, revealing both local and global structures. The method is based on the manifold learning assumption, which assumes that the dataset consists of sample points distributed over a smooth manifold, and uses the Laplace-Beltrami operator to capture the topological information of the data through the diffusion process.

Vector Diffusion Maps (VDM) \cite{singer2012vector} extend the theory of Diffusion Maps by replacing real-valued function weights with vector-valued functions. This approach captures connectivity by considering linear orthogonal transformations that encode changes of basis between tangent spaces at different data points, while simultaneously approximating parallel transport. By incorporating these geometric relationships, VDM extract richer structural information from the dataset. The theory of VDM has been applied in various fields, including cryo-electron microscopy \cite{singer2011three,TOADER2023168020}. The methodology is rooted in the connection Laplacian, which operates on vector fields and is approximated using a discrete formulation of the connection Laplacian operator. This operator is related to the first-order Hodge Laplacian via the Weitzenböck identity.  

The $k$-th Hodge Laplacian generalizes the Laplace-Beltrami operator and plays a crucial role in capturing the topology of a manifold through the $k$-th De Rham cohomology group. Recent studies have extended the Hodge Laplacian to graphs and combinatorics, applying it to various fields such as ranking, game theory \cite{ribando2024combinatorial, jiang2011statistical, lim2020hodge}, and biomolecular structure analysis \cite{wei2022hodge}. In this paper, we aim to approximate the Hodge Laplacian, defined over a real manifold. 

While Vector Diffusion Maps leverage the first-order Hodge Laplacian to extract geometric information from data, an open question remains, as posed by \cite{singer2012vector}: How can higher-order geometric structures of a dataset be captured using the Hodge Laplacian? One of the main goals of this paper is to address this question. 

In this work, we introduce \textbf{Hodge Diffusion Maps (HDM)}, a novel extension of Vector Diffusion Maps that utilizes the $k$-th order Hodge Laplacian for any $k\geq1$. This approach overcomes the limitations of traditional methods such as Diffusion Maps and Vector Diffusion Maps, which primarily focus on low-order geometric features and often miss important higher-order structures. HDM provides a powerful framework for nonlinear dimensionality reduction that preserves the intrinsic geometry of the data. By projecting the dataset onto the dot product of the leading eigenforms of the Hodge Laplacian, our method captures meaningful geometric patterns and reveals the underlying structure of the data at multiple scales

In this paper, we first construct the  Hodge Laplacian over a Riemanian manifold  by inferring the exterior derivative operator on differential forms. The main technical contribution is the approximation of the exterior derivative, defined over differential forms, using sample points distributed on an unknown manifold with an unknown intrinsic geometry. We provide analytical estimates for this approximation, which in turn allow us to approximate the Hodge Laplacian operator. This approximation generalizes the gradient operator approximation presented in \cite{10227282}, which uses asymmetric kernels to infer the diffusion properties of the dataset \cite{gomez2021diffusion,he2023diffusion,he2023learning}. We summarize the key contributions of this work as follows:

\begin{itemize}
    \item We propose an approximation of the exterior derivative operator based on sample point distributions, with the construction of the approximation independent of the dataset's distribution. Error bound estimates for this approximation are provided in \cref{teoremaprincipal}.
    
    \item Based on this exterior derivative approximation, we construct a sample-based approximation of the Hodge Laplacian operator acting on differential forms defined over the manifold representing the dataset.
    
    \item Using the approximation of the Hodge Laplacian, we introduce the Hodge Diffusion Maps algorithm, which projects the dataset onto the dot products of the eigenforms of the Hodge Laplacian. This methodology extends the Vector Diffusion Maps algorithm, which is defined over first-order differential forms, to higher-order differential forms for \( k \geq 1 \).
\end{itemize}

The paper is structured as follows: In \cref{mapasdedifusion}, we  briefly review the theory of Vector Diffusion Maps and the Hodge theory for real manifolds. \cref{derivadaexterior} details how to approximate the exterior derivative operator using a sample collection of points, as presented in \cref{teoremaprincipal}. In \cref{matrixcomputationfromula}, we apply \cref{teoremaprincipal} to give a matrix-based approximation of the exterior derivative operator and provide the numerical implementation in \cref{algoHMH2}. \cref{laplacianos} builds upon the results in \cref{matrixcomputationfromula} to compute the Hodge Laplacian operator and define the Hodge Diffusion Maps and the Hodge Diffusion distance. \cref{experimentos} presents numerical experiments on synthetic data, comparing the proposed methodology against Diffusion Maps, PCA, and t-SNE algorithms. \cref{conclusiones} presents the conclusions of the paper, highlighting future directions and potential applications of the proposed methodology. Finally, \cref{aprendiceformas,apendicepruebaprincipal} provide the technical details related to the proof of the main result in \cref{teoremaprincipal}.

\section{Preliminaries}
\label{mapasdedifusion}
Throughout this paper, we denote by $\mathcal{M}$ a closed (i.e., compact without boundary) Riemannian manifold of dimension $d$, embedded in the ambient space $\mathbb{R}^n$. For a detailed exposition of diffusion map theory, we refer the reader to \cite{coifman2006diffusion,lafon2004diffusion}, and for an introduction to Hodge Laplacians, to \cite{warner1983foundations}. 
 
\subsection{Diffusion Maps}

We briefly explain the Diffusion Maps and Vector Diffusion Maps algorithms. Given a set of data points, \( X = \{x_1, x_2, \dots, x_N\} \subseteq \mathcal{M} \), the algorithm follows these steps:

First, we create a graph by measuring the similarity between pairs of data points. In classical diffusion maps, the weight \( g_{ij} \) of the edge between points \( x_i \) and \( x_j \) is calculated using the Gaussian Kernel:
    \[
    g_{ij} = \exp\left(-\frac{\|x_i - x_j\|^2}{\epsilon}\right)
    \]
    where \( \epsilon \) is a scaling parameter. In the case of vector diffusion maps, the weights \( \overline{g}_{ij} \) are calculated using an orthogonal transformation \( ORT_{ij} \) along with the Gaussian kernel:
    \[
    \overline{g}_{ij} = ORT_{ij}  g_{ij}
    \]
The next step is to normalize the similarity matrix. In classical diffusion maps, the normalization is given by:
    \[
    P_{ij} = \frac{g_{ij}}{\sum_k g_{ik}}
    \]
    In vector diffusion maps, the normalization is:
    \[
    P_{ij} = ORT_{ij} \frac{g_{ij}}{\sum_k g_{ik}}
    \]
    This matrix \( P \) represents the probabilities of moving from one point to another in the diffusion process. In the final step, the algorithm then computes the eigenvalues and eigenvectors of the matrix \( P \). Let \( \phi_1, \phi_2, \dots, \phi_k \) be the eigenvectors corresponding to the largest eigenvalues. The diffusion maps (and vector diffusion maps) project the data points into a lower-dimensional space based on the entries of these eigenvectors:
    \[
    \psi_i = \left(\lambda_1 \phi_1(i), \lambda_2 \phi_2(i), \dots, \lambda_k \phi_k(i)\right)
    \]
    where \( \lambda_i \) are the eigenvalues and \( \phi_i(i) \) is the \( i \)-th entry of the eigenvector \( \phi_i \).
This process reduces the data’s dimensionality while maintaining its intrinsic geometric structure. Additionally, the matrices \( P_{ij} \) approximate the Laplace-Beltrami and Connection Laplacian operators.
    
In comparison with classical diffusion maps and vector diffusion maps, our proposed method considers the spectral decomposition of matrices of the form:
\[
 P_{ij} = \frac{1}{\sum_k g_{ik}} \det\left[g_{ij} L_i (x_i - x_j)^T, L_{ij}\right]
\]
 where \( L_i \) and \( L_{ij} \) are linear transformations depending on the indices \( i \) and \( j \), respectively. We use this spectral decomposition to extract topological information from the dataset. Additionally, we show that the form of these matrices can approximate the exterior derivative operator, which is explained in more detail in Section \cref{matrixcomputationfromula}.

\subsection{Hodge Laplacians}
\label{sub:2Hodge}

The fundamental object in the Hodge theory for a real manifolds $\mathcal{M}$ are the Hodge Laplacians, sometimes also called Laplace–de Rham operator. Hodge Laplacians $\Delta^k$ are linear operators defined over the set $k$-differential forms $\Omega^k(\mathcal{M})$.
The set of Hodge Laplacians generalizes the Laplace-Beltrami operator $\Delta$ in the sense that for $k=0$, the two notions of Laplacian coincide $\Delta^0=\Delta,$ up to a sign. The importance of these operators lies in the fact that their kernels correspond to algebraic invariants that encode geometric and topological information about the manifold. Let us briefly recall the definition of the Hodge Laplacian. Let $(\M,g)$  be an oriented Riemannian manifold of dimension $d$. A Riemannian metric on a smooth manifold is a smooth assignment of an inner product to each tangent space. The Riemannian metric $g$ induces an isomorphism $T_x\M\simeq T_x^*\M$ for every $x\in\M,$ allowing the inner product in $T_x\M$ to be naturally transferred to $T_x^*\M$. This inner product in $T_x^*\M$ extends to the exterior algebra $\bigwedge^kT_x^*\M$ via the determinant:

\begin{equation}
\label{eq:2403}
\langle w_1\wedge\cdots\wedge w_k,v_1\wedge\cdots\wedge v_k \rangle:=\text{det}[\langle w_i,v_j \rangle], \;\;w_i,v_j\in T_x^*\M.
\end{equation}
Using the metric and the orientation, one defines the Hodge star operator 
$$
\star:\Omega^{k}(\M)\to \Omega^{d-k}(\M) 
$$
which, for a $k$-differential form $\omega$, is uniquely determined by the relation 
$$
\eta\wedge(\star\omega)=\langle \eta,\omega\rangle dVol
$$
for every $k$-differential form $\eta$, where $\langle \eta,\omega\rangle$ is the pointwise inner product defined by \cref{eq:2403} and $dVol$ is the volume form induced by the Riemannian metric $g$. The adjoint of the exterior derivative, $\mathbf{d}_k^{*}:\Omega^{k}(\M)\to\Omega^{k-1}(\M)$,  is given by

$$
\mathbf{d}_k^{*}:=(-1)^{d(k+1)+1}\star\mathbf{d}_{d-k}\star. 
$$
The Hodge Laplacian is then defined as 
$$
\Delta^k:=\mathbf{d}_{k+1}^{*}\mathbf{d}_k+\mathbf{d}_{k-1}\mathbf{d}_k^{*}
$$
which is an endomorphism of $\Omega^{k}(\M).$  The Hodge Laplacian provides important information about the cohomology elements, which intuitively correspond to $k$-dimensional holes. This follows from the Hodge theorem, which states that the space of $k$-harmonic forms  
$$\mathcal{H}^k(\mathcal{M}):=ker(\Delta^k)$$
is isomorphic to the $k$-th de Rham cohomology group:
\begin{equation*}
     H_{dR}^{k}(\mathcal{M}) \simeq\mathcal{H}^k(\mathcal{M}).
\end{equation*}
Thus, the Hodge Laplacian, which is constructed using the exterior derivative \(\mathbf{d}_{k}\), encodes topological information about the manifold. In this paper, we focus on inferring the exterior derivative from an observable set of sample points distributed over the manifold.

\section{Exterior derivative approximation}
\label{derivadaexterior}
In this section, we extend the Diffusion Maps method from smooth functions to \( k \)-differential forms, aiming to approximate the exterior derivative operator on a manifold \(\mathcal{M}\) using sample points from \(\mathcal{M}\).  Our approach builds upon and generalizes the gradient estimation introduced in \cite{10227282}. We assume that \(\mathcal{M}\) is a compact, \(d\)-dimensional Riemannian submanifold of \(\mathbb{R}^n\) without boundary, where the Riemannian metric on \(\mathcal{M}\) is induced by the ambient space.

\subsection{Differential forms and differential arrays}

We recall that a $k$-differential  form is a smooth section $\omega: \mathcal{M} \to \bigwedge^k T^*\mathcal{M} $  defined from the manifold $\mathcal{M}$ to the $k$-th exterior power of the cotagent bundle $T^*\mathcal{M}$, such that for every $x \in \mathcal{M}$, the value of $\omega$ at $x$ is an element $\omega_x \in \bigwedge^k T^*_{x}\mathcal{M}.$ In other words, at each point $x\in\mathcal{M},$ $\omega_x$ is a linear functional 

$$\omega_x:\bigwedge^k T_x\mathcal{M} \to  \R,$$
or equivalently, an $k$-alternating form on the tangent space $T_x\mathcal{M},$

$$\omega_x:\underbrace{T_x\mathcal{M}\times \cdots \times T_x\mathcal{M}}_{k-times} \to \mathbb{R}.$$
See \cref{def:form2811} for further details. The set of all $k$-differential forms on $\mathcal{M}$ is denoted by $\Omega^k( \mathcal{M})$. Now, consider a local coordinate system $(v_1,\cdots,v_d)$ on $\mathcal{M}.$ In these coordinates, any $k$-differential form $\omega\in\Omega^k(\mathcal{M})$ can be expressed as    

$$\omega= \sum_{I} a_I dv_{i_1} \wedge  \cdots \wedge dv_{i_k}.$$
The exterior derivative $\mathbf{d}_k: \Omega^k(\mathcal{M}) \to \Omega^{k+1}(\mathcal{M}),$ acts on differential forms, and in these coordinates, it is given by
\begin{equation*}
  \mathbf{d}_k \omega = \mathbf{d}_k\paren{\sum_{I} a_I dv_{i_1} \wedge  \cdots \wedge dv_{i_k}}= \sum_I \sum_j \frac{\partial a_I}{\partial v_{j}}  dv_{j} \wedge dv_{i_1} \wedge  \cdots \wedge dv_{i_k},
\end{equation*}
While differential forms are abstract mathematical objects, we emphasize that, for computational purposes, we will represent them using alternating arrays, see \cref{def:alternatingmulti} . From \cref{proposicion:universal}, we know that $\bigwedge^k(T_x^*\mathcal{M})\simeq \Theta^k(T_x\M)$, allowing us to introduce the following space: 

\begin{equation*}
    \Theta^{k}(T\mathcal{M})= \bigsqcup_{ x \in \mathcal{M}} \Theta^{k}( T_{x}\mathcal{M})
\end{equation*}
which admits a vector bundle structure such that $\bigwedge^k(T^*\M)\simeq \Theta^{k}( T\mathcal{M})$; see \cite[Lemma 10.6]{lee2012introduction} for technical details. A section $W: \mathcal{M} \to \Theta^{k}( T\mathcal{M}),$ called a $k$-differential array, is, via this isomorphism, simple a differential form on $\M$. We denote the space of $k$-differential arrays on $\M$ by $\Theta^{k}(\mathcal{M})$, which corresponds to the space of differential forms $\Omega^{k}(\M),$ where each $W_x$ is regarded as a $k$-alternating array in $T_x\M$. See \cref{appeA} for a detailed explanation of this identification.  

From now on, we will use differential arrays, but the reader should keep in mind that they are fundamentally differential forms, represented in a way that is more suitable for numerical computations.

\subsection{Approximation of the exterior derivative using sample points}

The goal of this section is to approximate the exterior derivative using sample points distributed on a manifold $\mathcal{M}$ according to a smooth density $q(x)$. To achieve this, we first consider the heat kernel on $\mathcal{M} \times \mathcal{M}$, given by
\begin{equation*}
   e^{\frac{-\| y-x \|^2}{2 t^2}},
\end{equation*}
where \(\| \cdot \|\) denotes the Euclidean norm. We then introduce the following normalization: 
\begin{equation}
\label{funcionnormalizacion}
    d_t(x)=\int_{\mathcal{M}} e^{\frac{-\| y-x \|^2}{2 t^2}} q(y) dVol(y),
\end{equation}
where $dVol$ is the volume form induced by the Riemannian metric. Using this, we define the asymmetric vector-valued kernel
\begin{equation}
\label{ecuacionkernelvec}
   K_{t}(x,y)=(y-x)\frac{e^{\frac{-\| y-x \|^2}{2 t^2}}}{d_t(x)}.
\end{equation}
For fixed $x,y,t,$ each vector $K_{t}(x,y)\in\R^n$  is a $1$-differential array and, under the respective identification, can also be regarded as a $1$-alternating array. Consequently, for any $W\in\Theta^{k}(\mathcal{M})$, it makes sense to consider
$$
K_t(x,y) \wedge W(x))(n_1, n_2, \cdots, n_{k+1})
$$
for fixed $x\in\mathcal{M}$, which defines a $(k+1)$-alternating arrays on $\R^n$. Now, for every differential array $W\in\Theta^{k}(\mathcal{M})$, we define a differential array $\mathbf{P_t}W\in\Theta^{k+1}(\R^n)$ by

\begin{equation*}
    \begin{array}{lrc}
\mathbf{P_t} W(x)(n_1, n_2, \dots, n_{k+1}) = \\
\int_{\mathcal{M}}\left(K_t(x,y) \wedge (W(y) - W(x))\right(n_1, n_2, \cdots, n_{k+1}) q(y) \, dVol(y), 
    \end{array}
\end{equation*}
where \(1 \leq n_i \leq n\). Here, the integral is well-defined since for fixed $x$ and $(n_1,\dots,n_{k+1}),$ 
$$\left(K_t(x,y)\wedge (W(y) - W(x))\right(n_1, n_2, \cdots, n_{k+1})$$
can be regarded as a real-valued function of $y$ defined on $\M.$ Note that in this integral, we interpret $W(y)$ as being defined in the same space as $W(x)$, namely on $T_x \M.$ Thus, $\mathbf{P_t}$ defines a linear operator on $\Theta^{k}( \mathcal{M})$. For any $W\in\Theta^k(\M)$ and $x\in\M$, we denote     
\begin{equation}
\label{operadorP}
    \mathbf{P_t} W (x)= \int_{\mathcal{M}} K_t(x,y) \wedge (W(y)-W(x)) q(y) dVol(y)
\end{equation}

\begin{rem}
     $\mathbf{P_t}W(x)$ defines a $(k+1)$-alternating array on $\R^n$ but not necessary a $(k+1)$-alternating array on $T_x \mathcal{M}$. However, according to \cref{observacionprojectionalternada}, its orthogonal projection onto $\bigwedge^{k+1} T_x\mathcal{M}$, denoted by 
     $$\mathcal{P}_{\bigwedge^{k+1} T_x \mathcal{M}} (\mathbf{P_t} W (x))$$
    defines a $(k+1)$-alternating form on $T_x \mathcal{M}.$  
\end{rem}

The following theorem establishes the relation between the operator $\mathbf{P_t}$ and the exterior derivative over $k$-differential arrays.

\begin{thm}
   \label{teoremaprincipal}
   Let $W \in \Theta^k(\mathcal{M})$ be a $k$-differential array, and let $x \in \mathcal{M}.$ For any $\delta$ satisfying

   \begin{equation}
   \label{deltacondicion}
    \frac{1}{2} < \delta < 1- \frac{d}{2(d+2)}<1.
\end{equation}
where $d$ is the dimension of $\M,$ the following estimate holds:
   \begin{equation*} 
       \mathcal{P}_{\bigwedge^{k+1} T_x \mathcal{M}} (\mathbf{P_t}W(x)) = t^2\left(\mathbf{d}_kW(x)\right) +O(t^f)
   \end{equation*}
where the exponent $f$ is given by  

$$\min \llave{4 \delta -2,2(1-\delta)(d+2)}.$$ 
In particular, taking the limit as $t\to 0^+,$ we obtain

 \begin{equation}
 \label{ecuacionprincipal}
      \lim_{t \to 0^{+}}\frac{1}{t^2} \mathcal{P}_{\bigwedge^{k+1} T_x \mathcal{M}} (\mathbf{P_t} W (x)) = \mathbf{d}_k W(x).
 \end{equation}
\end{thm}

First, we note that the set of values for $\delta$ satisfying \cref{deltacondicion} is nonempty. Indeed, since $0< \frac{d}{2(d+2)} < \frac{1}{2}$, it follows that

$$\frac{1}{2}< 1- \frac{d}{2(d+2)} <1.$$
\cref{ecuacionprincipal} provides a method for estimating the exterior derivative $\mathbf{d}_k W$ 
based on a set of sample points observed on the manifold. Specifically, for small $t>0$, we have

 \begin{equation*}
 %\label{ecuacionaproximacion}
     \frac{1}{t^2}\mathcal{P}_{\bigwedge^{k+1} T_x \mathcal{M}} (\mathbf{P_t} W (x)) \approx  \mathbf{d}_k W(x).
 \end{equation*}
By the Law of Large Numbers (LLN), the operator \( \mathbf{P_t} \) can be approximated as:

\begin{equation*}
    \mathbf{P_t} W (x) \approx \sum_{i=1}^N \bar{K}_t(x,x_i) \wedge (W(x_i)-W(x)),
\end{equation*}
where

\begin{equation}
\label{ecuacionkernelvec1}
   \bar{K}_{t}(x,y)=(y-x)\frac{e^{\frac{-\| y-x \|^2}{2 t^2}}}{\overline{d_t}(x)},\;\text{ and }\overline{d_t}(x)=\sum_{i=1}^N e^{\frac{-\| x_i-x \|^2}{2 t^2}}.
\end{equation}
Consequently, the exterior derivative at \( x_j \) can be estimated as:
\begin{equation}
\label{approximacionderivadaexterior}
  \mathbf{d}_k  W(x_j) \approx   \mathcal{P}_{\bigwedge^{k+1}T_x \mathcal{M}}\paren{\frac{1}{t^2}\sum_{i=1}^N \bar{K}(x_j,x_i)\wedge \left(W(x_i)-W(x_j)\right) }
\end{equation}
Observe that, by the Law of Large Numbers, both terms \( \mathbf{P_t} \) and \( \overline{d_t}(x) \) should have an average factor of \( 1/N \), but the variable \( N \) cancels out in the division involved in approximating the exterior derivative \( \mathbf{d}_k W(x_j) \).

Notably, the right-hand side of \cref{approximacionderivadaexterior} does not depend on the distribution $q(x)$. Instead, it is computed purely from the dataset $x_1, x_2, \dots, x_N$. This result enables the estimation of the exterior derivative independently of the underlying distribution $q(x),$ making it a robust method for data-driven differential analysis.

\section{Matrix-based computations} \label{matrixcomputationfromula}
In this section, we derive a matrix representation for the approximation of the exterior derivative given in \cref{approximacionderivadaexterior}. To achieve this, we first express the space of $k$-differential arrays in matrix form. The section is 
structured as follows: in \cref{differentialmatrix}, we present the matrix representation of the set of $k$-differential arrays \( \Theta^k(\mathcal{M}) \). Then, in \cref{computationmatrixextderivative}, we present a matrix formulation for computing the exterior derivative of a $k$-differential array. This matrix-based approach provides a practical framework for numerical implementation and analysis.

\subsection{Matrix representation of differential arrays} 
\label{differentialmatrix}

A key question in the proposed approach is how to reconstruct the space of $k$-differential arrays, denoted by \(\Theta^k(\mathcal{M})\), from a finite set of \(N\) sample points \(X= \{ x_1,x_2, \dots, x_N \}\). These points are realizations of $N$ independent and identically distributed (i.i.d.) random variables \(X_1, X_2, \dots, X_N\) drawn from a smooth \(q(\cdot)\) over an \textbf{unknown} \(d\)-dimensional manifold \(\mathcal{M}\). To address this, we first describe the local construction of $k$- differential arrays. Given a \( k \)-differential array \( W \), its evaluation at \( x_i \) can be expressed as

\begin{equation}
\label{localW}
   W(x_i)= \sum_{J} f_{J}(x_i) O_{J}(x_i),
\end{equation}
where the sum is taken over all \(k\)-tuples \(J = (j_1, j_2, \dots, j_k)\) with \(1 \le j_1 < j_2 < \dots < j_k \le d\), Here, \(f_{J}(x_i)\) are real-value function, and $O_J(x_i)$ is the orthonormal basis of \(\Theta^k(T_{x_i}\mathcal{M}) \) defined as the wedge product
\begin{equation}
\label{definicionnuevaO}
    O_{J}(x_i)= \frac{1}{\sqrt{k!}} O_{j_1}(x_i) \wedge \cdots \wedge O_{j_k}(x_i), 
\end{equation}
where $\{O_{1}(x_i), \cdots , O_{d}(x_i)\}$ is a orthonormal basis for the tangent space $T_{x_i}\M$. This orthonormal basis is constructed using the Local PCA methodology described in  \cite{singer2012vector} and \cite{singer2011orientability}. The local PCA algorithm computes an orthonormal basis for the tangent space \(T_{x_i}\mathcal{M}\) and the dimension of the manifold $\M$ as follows: 
\begin{enumerate}
    \item \textbf{Neighborhood Selection:} Given a positive parameter $r$, consider the set of points \(\{ x_{i_1}, x_{i_2}, \dots, x_{i_l} \}\)
    that lie within the local neighborhood
    \begin{equation*}
    U(x_i, r)=\{ y\in\M | \|x_i-y\|_{\R^n}< r \}.
\end{equation*}

\item \textbf{Matrix construction:} Define the Matrix $M_{x_i}$ as

\[
   M_{x_i} = \left[ (x_{i_1} - x_i) e^{\frac{\|x_{i_1} - x_i\|^2}{2t^2}}, (x_{i_2} - x_i) e^{\frac{\|x_{i_2} - x_i\|^2}{2t^2}}, \dots, (x_{i_l} - x_i) e^{\frac{\|x_{i_l} - x_i\|^2}{2t^2}} \right].
\]
\item\textbf{Estimating the Intrinsic Dimension:} 
Let \(\sigma_1, \sigma_2, \dots, \sigma_{\min(i_l, n)}\) denote the singular values of \(M_{x_i}\). We introduce a threshold parameter \(\gamma \in (0,1)\), typically set to \(\gamma \approx 0.9\), and define the intrinsic dimension \(d_i\) at \(x_i\) as the largest integer satisfying  
\[
    \frac{\sum_{j=1}^{d_i} \| \sigma_j \|}{\sum_{j=1}^{i_l} \| \sigma_j \|} < \gamma.
\]  
The parameter \(d_i\) provides an estimate of the dimension of the local tangent space at \(x_i\). The intrinsic dimension \(d\) of the manifold \(\mathcal{M}\) is then obtained as the median of all local estimates:  
\[
    d = \operatorname{median}(d_1, d_2, \dots, d_N).
\]  

\item \textbf{Extracting the Tangent Space Basis:} Compute the singular value decomposition (SVD) of $M_{x_i},$ and take the first $d$ left-singular vectors $O_1(x_i), \dots, O_d(x_i) $ as an orthonormal basis for the tangent space \(T_{x_i}\mathcal{M}\). 
\end{enumerate}

The local PCA algorithm allows  to express a $k$-differential array $W$ in matrix form as 

\begin{equation*} 
%\label{matrixbloquesW}
    \textbf{O}_k*\textbf{f}
\end{equation*}
where $*$ denotes the standard matrix multiplication, and the matrices $\textbf{O}_k$ and $\textbf{f}$ are defined as follows:\\
\textbf{Definition of $\mathbf{f}$:} The matrix $\mathbf{f}$ consists of $N$ blocks, each of size $\binom{d}{k}\times 1.$ The $i$-th block is given by  
\begin{equation*}
    \textbf{f} (i)=\begin{bmatrix}
f_{J_1}(x_i)\\
f_{J_2}(x_i)\\
\vdots\\
f_{J_{\binom{d}{k}}}(x_i)
\end{bmatrix}
\end{equation*}
where the multi-indexes $J_1, J_2, \dots \ J_{\binom{d}{k}}$ correspond to all possible $k$-tuples
$$J_l=(j^l_1, \dots, j^l_k) \text{ with } 1 \le j^k_1 < \dots < j^l_k \le d.$$ 
Thus, the full matrix \( \textbf{f} \) has size \( \binom{d}{k} N\times 1 \).\\
\textbf{Definition of $\textbf{O}_k$:} The matrix \( \textbf{O}_k \) consists of $N\times N$ blocks, each of size $n^k\times \binom{d}{k}.$ The block at position $(i,j)$ is defined as

\begin{equation*}
    \textbf{O}_k(i,j)= \begin{cases}
			\overline{O}_{k}(i)  & \text{if $i=j$ }\\
            0_{ n^k\times \binom{d}{k}}, & \text{if $i \neq j$ }
		 \end{cases}
\end{equation*}
for $i,j\in\{1,\dots,N\}$. Here, $\overline{O}_{k}(i)$ is the matrix
\begin{equation*}
   \overline{O}_{k}(i)=\begin{bmatrix} O_{J_1}(x_i) & O_{J_2}(x_i) & \cdots & O_{J_{\binom{d}{k}}}(x_i)
   \end{bmatrix},
\end{equation*}
where each $O_{J_l}(x_i)$ is defined as in  \cref{definicionnuevaO} and is considered as a column vector embedded in $\R^{n^k}$. Overall, $\textbf{O}_k$ has size $n^k N\times \binom{d}{k} N.$ The values of the $k$-differential array $W(x_i)$ correspond to the $i$-th block of the product $\textbf{O}_k*\textbf{f}.$

\subsection{Matrix-based computation of the exterior derivative}
\label{computationmatrixextderivative}

In this section, we derive a matrix expression for the approximation of the exterior derivative in $k$-differential arrays, as given in \cref{approximacionderivadaexterior}, using the results from the previous section. According to  \cref{localW}, any  \( k \)-differential array \( W \) at the point \( x_j \) can be written as:
\[
W(x_j) = \sum_{J} f_{J}(x_j) O_{J}(x_j).
\]
We denote by $O(x_j)_{n\times d}$ the matrix whose columns form  a basis for the tangent space \( T_{x_j} \mathcal{M} \) given by
\[
O_1(x_j), O_2(x_j), \dots, O_d(x_j).
\]
Next, the projection of this basis onto the tangent space at the point \( x_i \), denoted \( T_{x_i} \mathcal{M} \), is given by the matrix product:
\[
O(x_i) O(x_i)^T  O(x_j).
\]
Let \( \mathcal{P}_V \) denote the orthogonal projection onto the space \( V \). Then, $O_J(x_j)$ decomposes as
\[
\begin{array}{rcl}
    O_J(x_j) &= & \frac{1}{\sqrt{k!}} O_{j_1}(x_j) \wedge \cdots \wedge O_{j_k}(x_j) \\
    & = & \frac{1}{\sqrt{k!}} \mathcal{P}_{T_{x_i} \mathcal{M}} O_{j_1}(x_j) \wedge \cdots \wedge \mathcal{P}_{T_{x_i} \mathcal{M}} O_{j_k}(x_j) + \xi,
\end{array}
\]
where $\xi$ consists of wedge product terms in which one factor of each wedge product belongs to the orthogonal complement $T_{x_i} \mathcal{M}^\perp$. Furthermore, the projected term
\[
\mathcal{P}_{T_{x_i} \mathcal{M}} O_{j_1}(x_j) \wedge \cdots \wedge \mathcal{P}_{T_{x_i} \mathcal{M}} O_{j_k}(x_j)
\]
can be written as:
\[
\sum_{L} \det(O_L^T(x_i) O^J(x_j)) O_{l_1}(x_i) \wedge O_{l_2}(x_i) \wedge \cdots \wedge O_{l_k}(x_i),
\]
where \( O_L^T(x_i) \) denotes the submatrix of \( O(x_i)^T \) formed by the rows indexed by $L=(l_1,\dots,l_k)$, while \( O^J(x_j) \) is the submatrix of \( O(x_j) \) consisting of the columns indexed by $J=(j_1,\dots,j_k)$. By Combining this with \cref{localW}, we obtain that the orthogonal projection onto $\bigwedge^k T_{x_i}\M$ is given by
\begin{equation}
\label{eqderext1}
\begin{split}
\mathcal{P}_{\bigwedge^k T_{x_i} \mathcal{M}} W(x_j)=\sum_{J} f_{J}(x_j) \sum_{L} \det(O^{T}_{L}(x_i) O^{J} (x_j)) O_{L}(x_i)
\end{split}
\end{equation}
The previous equation helps to implement  \cref{teoremaprincipal} as follows. According to \cref{approximacionderivadaexterior}, we can approximate the exterior derivative $ \mathbf{d}_k(W) (x_i)$ as
\begin{equation}
\label{eqderext2}
\begin{split}
    \frac{1}{t^2} \mathcal{P}_{\bigwedge^{k+1}T_{x_i} \mathcal{M}} &\paren{ \sum_{j=1}^N \bar{K}_t(x_i,x_j) \wedge \left(W(x_j)-W(x_i)\right)}= \\
   &\frac{1}{t^2} \sum_{j=1}^N \paren{ \mathcal{P}_{T_{x_i} \mathcal{M}} ( \bar{K}_t(x_i,x_j) ) \wedge  ( \mathcal{P}_{\bigwedge^{k}T_{x_i}\mathcal{M}}\left(W(x_j)- W(x_i)\right)}
   \end{split}
\end{equation}
Observe that by definition of the kernel  $\bar{K}_t(x_j,x_i)$ as in \cref{ecuacionkernelvec1}:
\begin{equation}\label{eqderext3}
\mathcal{P}_{T_{x_i} \mathcal{M}}( \bar{K}_{t}(x_i,x_j))=\frac{1}{\overline{d}_t(x_i)}  e^{\frac{-\| x_i-x_j \|^2}{2 t^2}} \sum_{s=1}^N \langle x_j-x_i , O_s(x_i)\rangle  O_s(x_i).
\end{equation}
Using the Laplace expansion of the determinant and the identity
\begin{equation*}
   O_s(x_i) \wedge O_{L}(x_i)=  \frac{1}{\sqrt{k!}}  O_s(x_i)  \wedge O_{l_1}(x_i) \wedge \cdots \wedge O_{l_k}(x_i) =\sqrt{k+1} \, \, O_{(s, l_1, \cdots ,l_k)}(x_i),
\end{equation*}
we obtain the following expression:
\begin{equation}
\label{eqderext4}
\begin{split}
    \sum_{s=1}^N \langle x_j-x_i , O_s(x_i)\rangle  O_s(x_i) &\wedge  \sum_{L} \det(O_{L}^T(x_i) O^J (x_j)) O_{L}(x_i)=\\
    & \sqrt{k+1}  \sum_{M} \det ([A_{M}(i,j) , O_{M}^{T}(x_i) O^J (x_j)]) O_{M}(x_i)
\end{split}
\end{equation}
where the sum runs over all \( k+1 \)-tuples \( M = (m_1, m_2, \dots, m_{k+1}) \) satisfying \( 1 \le m_1 < \dots < m_{k+1} \le d \). Here \( A(i,j) \) is the column vector defined by
\begin{equation*}
  A(i,j)=  e^{\frac{-\| x_i-x_j \|^2}{2 t^2}}O(x_i)^{T}(x_j-x_i)
\end{equation*}
and \( A_M(i,j) \) is the submatrix of \( A(i,j) \) consisting of the rows indexed by \( M \). Additionally,
\begin{equation}\label{matrixAalgo}
    [A_{M}(i,j) , O_{M}^{T}(x_i) O^L (x_j)]
\end{equation}
denotes the concatenated matrix whose first column is \( A_M(i,j) \). Therefore, combining \cref{eqderext1,eqderext3,eqderext4}, we obtain the following wedge product identity:
\begin{equation}
\label{ecuacion1bonita}
\begin{split}
    &\sum_{j=1}^N \paren{ \mathcal{P}_{T_{x_i} \mathcal{M}} ( \bar{K}_t(x_i,x_j) ) \wedge  ( \mathcal{P}_{\bigwedge^k T_{x_i} \mathcal{M}} W(x_j)}=\\
    &\sqrt{k+1} \frac{1}{\overline{d}_t(x_i)}  \sum_{j=1}^N  \sum_{J} f_{J}(x_j) \sum_{M} \det ([A_{M}(i,j) , O_{M}^{T}(x_i) O^J (x_j)]) O_{M}(x_i).
    \end{split}
\end{equation}
Similarly
\begin{equation}
\label{ecuacion1bonita2}
\begin{split}
    &\sum_{j=1}^N \paren{ \mathcal{P}_{T_{x_i} \mathcal{M}} ( \bar{K}_t(x_j,x_i) ) \wedge  W(x_i)}=\\
    &\sqrt{k+1} \frac{1}{\overline{d}_t(x_i)}  \sum_{j=1}^N  \sum_{J} f_{J}(x_i) \sum_{M} \det ([A_{M}(i,j) , O_{M}^{T}(x_i) O^J (x_i)]) O_{M}(x_i)
    \end{split}
\end{equation}
Recall that \cref{eqderext2} provides an approximation of the exterior derivative $\mathbf{d}_k (W)(x_i)$. Note that, up to factor of $\frac{1}{t^2},$ \cref{eqderext2} corresponds to the difference between \cref{ecuacion1bonita} and \cref{ecuacion1bonita2}. 
Furthermore, \cref{ecuacion1bonita} represents the \( i \)-th block of the following matrix multiplication:
\begin{equation}
\label{ecuacionmatricial1}
 \sqrt{k+1}    \textbf{O}_{k+1}* \textbf{ED}^1_k*\textbf{f}
\end{equation}
where $\textbf{ED}^1_k$ is the block matrix
\begin{equation} \label{matrixHnegrita}
\textbf{ED}^1_k(i,j)=
    \begin{bmatrix}
       ED^1(i,j,M_1,J_1) & ED^1(i,j,M_1,J_2) & \cdots & ED^1(i,j,M_1,J_{\binom{d}{k}})\\
        ED^1(i,j,M_2,J_1) & ED^1(i,j,M_2,J_2) & \cdots & ED^1(i,j,M_2,J_{\binom{d}{k}})\\
       \vdots & \vdots & \ddots & \vdots \\
       ED^1(i,j,M_{\binom{d}{k+1}},J_1) & ED^1(i,j,M_{\binom{d}{k+1}},J_2) & \cdots & ED^1(i,j,M_{\binom{d}{k+1}},J_{\binom{d}{k}})
    \end{bmatrix}
\end{equation}
with entries defined as
\begin{equation*}
     ED^1(i,j,M,J)= \frac{1}{\overline{d}_t(x_i)} \det ([A_{M}(i,j) , O_{M}^{T}(x_i) O^J (x_j)])
\end{equation*}
Similarly, \cref{ecuacion1bonita2} represents the \( i \)-th block of the matrix multiplication.
\begin{equation}
\label{ecuacionmatricial2}
   \sqrt{k+1}  \textbf{O}_{k+1}* \textbf{ED}^2_k*\textbf{f}
\end{equation}
where $\textbf{ED}^2_k$ is the diagonal block matrix
\begin{equation*}
  \textbf{ED}^2_k(i,j)= \begin{cases}
			\overline{ED}_k^{2}(i)  & \text{if $i=j$ }\\
            0_{ \binom{d}{k+1}\times \binom{d}{k}}, & \text{if $i \neq j$ }.
		 \end{cases}
\end{equation*}
The block matrix $\overline{ED}^2_k(i)$ is given by  
\begin{equation}
\label{ecuacionmatricialnegrita2} 
\overline{ED}_k^{2}(i)=
    \begin{bmatrix}
       ED^2(i,M_1,J_1) & ED^2(i,M_1,J_2) & \cdots & ED^2(i,M_1,J_{\binom{d}{k}})\\
        ED^2(i,M_2,J_1) & ED^2(i,M_2,J_2) & \cdots & ED^2(i,M_2,J_{\binom{d}{k}})\\
       \vdots & \vdots & \ddots & \vdots \\
       ED^2(i,M_{\binom{d}{k+1}},J_1) & ED^2(i,M_{\binom{d}{k+1}},J_2) & \cdots & ED^2(i,M_{\binom{d}{k+1}},J_{\binom{d}{k}})
    \end{bmatrix}
\end{equation}
where each block is defined by
\begin{equation}
\label{matrixblocoh2}
\begin{array}{rcl}
     ED^2(i,M,J) & = &\frac{1}{\overline{d}_t(x_i)} \sum_{l=1}^N\det ([A_{M}(i,l) , O_{M}^{T}(x_i) O^J (x_i)])\\
     &=& \frac{1}{\overline{d}_t(x_i)} \det (\sum_{l=1}^N A_{M}(i,l) , O_{M}^{T}(x_i) O^J (x_i)])
\end{array}
\end{equation}
Now, recall that the $k$-differential array $W$ can be express ass
$$W=\textbf{O}_k*\textbf{f},$$
which imples that $$\textbf{f}=\textbf{O}^T_k*W.$$ 
By combining Equations \cref{eqderext2,ecuacionmatricial1,ecuacionmatricial2}, we obtain that the approximation of exterior derivative \(  \mathbf{d}_k(W) \) at the point \( x_i \), is given by the \( i \)-th block of the matrix multiplication:
\begin{equation}
%\label{repfinextderivative}
\frac{1}{t^2} \sqrt{k+1}    \textbf{O}_{k+1}* \textbf{ED}_k*\mathbf{f}=\frac{1}{t^2} \sqrt{k+1}    \textbf{O}_{k+1}* \textbf{ED}_k* \textbf{O}_{k}^{T}*W
\end{equation}
where $\textbf{ED}_k$ is defined as 
$$ \textbf{ED}_k=  \textbf{ED}^1_k- \textbf{ED}^2_k.$$ 
Thus, the matrix 
\begin{equation}
\label{repfinextderivative}
    \frac{1}{t^2} \sqrt{k+1} \textbf{O}_{k+1}* \textbf{ED}_k* \textbf{O}_{k} ^{T}
\end{equation}
represents the matrix approximation of the exterior derivative operator $\mathbf{d}_k$ acting on $k$-differential arrays.

Note that the matrix \( \textbf{O}_{k+1} \) has orthonormal columns and depends on the ambient space dimension \( n \), whereas \( \textbf{ED}_k \) encapsulates information about the manifold of dimension \( d \). Consequently, \( \textbf{ED}_k \) encodes more information about the intrinsic manifold through the exterior derivative \( \mathbf{d}_k \).

\begin{rem}
    An important observation is that if we choose a different orthonormal basis \( O^{'}_{j_1}(x_i), \dots, O^{'}_{j_k}(x_i) \), the associated matrix \( \textbf{ED}^{'}_k \), as given in \cref{repfinextderivative}, is equivalent to \( \textbf{ED}_k \), in the following sense:

\begin{equation}
\label{unicidadedequi}
    \textbf{ED}^{'}_k = ((\textbf{O}_{k+1}^{'})^{T}* \textbf{O}_{k+1})*\textbf{ED}_k * (\textbf{O}_k ^{T} * \textbf{O}_k^{'}).
\end{equation}
Here, the matrices \( (\textbf{O}_{k+1}^{'})^{T}* \textbf{O}_{k+1} \) and \( \textbf{O}^{T}_k* \textbf{O}_k^{'} \) are orthonormal, since the $(i,i)$ blocks of \( \textbf{O}_{k+1}^{'}, \textbf{O}_{k+1} \) and \( \textbf{O}_k^{'}, \textbf{O}_k \) form orthonormal bases for \( \Theta^{k+1} T_{x_i} \mathcal{M} \) and \( \Theta^{k} T_{x_i} \mathcal{M} \), respectively. Therefore, the matrix \( \textbf{ED}_k \) is unique, up to the change of basis induced by \(\left(\textbf{O}_{k+1}{'}\right)^{T}*\textbf{O}_{k+1} \) and \( \textbf{O}_k^{T}* \textbf{O}_k^{'} \).
\end{rem}

\subsection{Implementation of the Algorithm}

In this section, we summarize the results from the previous sections and outline a practical algorithm for analyzing data sets using the matrix representation of the exterior derivative, as defined in \cref{repfinextderivative} in \cref{matrixcomputationfromula}. The primary objective here is to compute the matrix \( \textbf{ED}_k=  \textbf{ED}^1_k- \textbf{ED}^2_k \) as specified in \cref{repfinextderivative}. 

We assume that \(X = \{ x_1, x_2, \dots, x_N \}\) are sampled points, representing \(N\) independent and identically distributed (i.i.d.) random variables \(X_1, X_2, \dots, X_N\), drawn from a smooth distribution \(q(\cdot)\) over an unknown \(d\)-dimensional manifold \(\mathcal{M}\).

The first step in the algorithm is to compute the tangent vectors \( O_1(x_j), \dots, O_d(x_j) \) using the Local PCA method described in \cite{singer2012vector} and \cite{singer2011orientability}. For this, we take as input the number \( K \), which represents the total number of points in the open neighborhood of a point \( x \), defined as:
\begin{equation}
\label{ecuacionvecindades}
    U(x, r) = \{ y \ | \ \| x - y \|_{\mathbb{R}^n} < r \}
\end{equation}
In the implemented algorithm, the number \( K \) is the same for all points and does not depend on the indices \( i \) or \( j \). \cref{algoPCA} summarizes the local PCA method, which is explained in \cref{differentialmatrix}.

\begin{algorithm}[H]
\begin{flushleft}
\textbf{input} Data-set \( X = \{ x_1, x_2, \dots, x_N \} \), and \( K \), the number of points in the neighborhood \( U(x_i, r) \) and scaling parameter $t$.
\begin{enumerate}
  \item \textbf{for} $i = 1 $ to $N$ \textbf{do}
   \begin{itemize}
    \item Find the \( K \)-closest points to \( x_i \), denoted \( x_{i_1}, x_{i_2}, \cdots, x_{i_K} \).
    \item Compute the matrix
    \begin{equation*}
           M_{x_i} = \left[ (x_{i_1} - x_i) e^{\frac{\|x_{i_1} - x_i\|^2}{2t^2}}, (x_{i_2} - x_i) e^{\frac{\|x_{i_2} - x_i\|^2}{2t^2}}, \dots, (x_{i_K} - x_i) e^{\frac{\|x_{i_K} - x_i\|^2}{2t^2}} \right].
    \end{equation*}
    \item Compute \( d_{x_i} \), the rank of the matrix \( M_{x_i} \).
     \end{itemize}
    \item \textbf{end for}
    \item Let $ d = \text{median}(d_1, d_2, \dots, d_N)$
    \item \textbf{for} $i = 1 $ to $N$ \textbf{do}
    \begin{itemize}
   \item  Let \( O_1(x_i), \dots, O_d(x_i) \) be the \( d \) left singular vectors from the singular value decomposition of \( M_{x_i} \).
   \item Compute the matrix $O(x_i)$ as
       \begin{equation*}
           O(x_i)=\left[ O_1(x_i), O_2(x_i), \cdots, O_K(x_i) \right].
    \end{equation*}
    \end{itemize}
    \item \textbf{end for}
\end{enumerate}
\textbf{return}  The orthonormal vectors \( O_1(x_i), O_2(x_i), \dots, O_K(x_i) \) of the tangent space \( T_{x_i}\mathcal{M} \), the matrix \( O(x_i) \) and the dimension $d$ of the manifold.
\end{flushleft}
 \caption{ Local PCA method}
 \label{algoPCA}
\end{algorithm}

The next step in the proposed method is to compute the matrix \( \textbf{ED}_k=  \textbf{ED}^1_k- \textbf{ED}^2_k\), as explained in \cref{computationmatrixextderivative}. Since the exterior derivative at point \( x_i \) depends only on information from the neighborhood \( U(x, r) \) (see \cref{ecuacionvecindades}), we can reduce the number of points required to construct the matrices \( \textbf{ED}^1_k \) and \( \textbf{ED}^2_k \), thereby lowering the computational complexity.

The key idea is that, for each index \( i \), we compute the block matrices \( \textbf{ED}_k(i,j) \) only if \( x_j \) is among the \( K \)-nearest points to \( x_i \). If \( x_j \) is not one of the \( K \)-nearest points, we set \( \textbf{ED}_k(i,j) = 0 \). Similarly, when computing \(  \textbf{ED}^2_k \), we calculate the \( i \)-th block (as shown in \cref{matrixblocoh2}) by summing over the \( K \)-nearest points to \( x_i \). Specifically, we compute \( ED^2(i,M,J) \) as:
\begin{equation}
\label{ecuacionh2algo}
    ED^2(i,M,J) = \frac{1}{\overline{d}_t(x_i)} \det \left( \sum_{l=1}^K A_M(i,i_l), O_M^T(x_i) O^J(x_i) \right),
\end{equation}
where \( i_1, \dots, i_K \) are the indices of the \( K \)-nearest points \( x_{i_1}, \dots, x_{i_K} \in X \) to \( x_i \). This simplification does not significantly affect the expression for \( \textbf{ED}_k=  \textbf{ED}^1_k- \textbf{ED}^2_k \), since the exponential term
\[
e^{-\frac{\|x_j - x_i\|^2}{2t^2}}
\]
vanishes  when \( x_j \) is far from \( x_i \). The computation of the matrix \( \textbf{ED}_k=  \textbf{ED}^1_k- \textbf{ED}^2_k \) is summarized in \cref{algoHMH2}.

\begin{algorithm}[H]
\begin{flushleft}
\textbf{input} Data-set \( X = \{ x_1, x_2, \dots, x_N \} \), and \( K \), the number of points in the neighborhood \( U(x_i, r) \), scaling parameter $t$.
\begin{enumerate}
  \item Apply Algorithm \ref{algoPCA} and assign $[O(x_i), d] \gets Local PCA (X, K)$
   \item Initialize the array $\textbf{ED}_k$ as a zero matrix with $N \times N$  blocks.
    \item \textbf{for} $i = 1 $ to $N$ \textbf{do}
    \begin{itemize}
     \item Find the \( K \)-closest points to \( x_i \), denoted \( x_{i_1}, x_{i_2}, \cdots, x_{i_K} \).
     \item For all \( 1 \le l \le K \), compute the column vector \( A(i, i_l) \) as in Eq \eqref{matrixAalgo}.
    \item \textbf{for} $l = 1 $ to $K$ \textbf{do}
    \begin{itemize}
    \item \textbf{If} \( i \neq i_l \)
    \item Assign the value  \( \textbf{ED}_k(i, i_l) \gets \textbf{ED}_k^1(i, i_l) \) as shown in Equation \eqref{matrixHnegrita}.
    \item \textbf{Else}
    \item Assign the value  \( \textbf{ED}_k(i, i) \gets \overline{ED}^{2}_k(i) \) based on Equations \eqref{ecuacionmatricialnegrita2} and \eqref{ecuacionh2algo}.
    \item \textbf{End If}
     \end{itemize}
     \item \textbf{end for}
    \end{itemize}
    \item \textbf{end for}
\end{enumerate}
\textbf{return}  The matrix $\textbf{ED}_k$, which contains the intrinsic information of the exterior derivative.
\end{flushleft}
 \caption{ Computation of \( \textbf{ED}_k=  \textbf{ED}^1_k- \textbf{ED}^2_k \)}
 \label{algoHMH2}
\end{algorithm}

\section{ Hodge Diffusion-Maps}
\label{laplacianos}
In this section, we use the approximation of the exterior derivative provided in \cref{derivadaexterior} to construct a matrix approximation based on observable sample points of the Hodge Laplacian operator. Additionally, we define the Hodge diffusion maps and the Hodge diffusion distance, which are generalizations of the Vector Diffusion Maps methodology \cite{singer2012vector}.

\subsection{Hodge Laplacians approximation}

Based on the results from \cref{derivadaexterior,matrixcomputationfromula}, and specifically from \cref{repfinextderivative}, we can approximate the exterior derivative $\mathbf{d}_k$ on  $k$-differential arrays using the matrix:
\begin{equation*}
   \frac{1}{t^2} \sqrt{k+1} \textbf{O}_{k+1}* \textbf{ED}_k* \textbf{O}_{k} ^{T}
\end{equation*}
With this approximation, we can also approximate the Hodge Laplacian $\Delta_k$, which is defined on $k$-differential arrays as:
\begin{equation*}
    \Delta_k=\mathbf{d}_{k+1}^{*} \circ \mathbf{d}_k + \mathbf{d}_{k-1} \circ \mathbf{d}_{k}^{*},
\end{equation*}
(see \cref{sub:2Hodge} for more details). It turns out that the matrix representation of the approximation of the adjoint $\mathbf{d}_k^*$ of the exterior derivative corresponds to the transpose of the matrix representation of $\mathbf{d}_{k-1}.$ In other words, this approximation is represented by the matrix: 
\begin{equation*}
   \frac{1}{t^2} \sqrt{k} \textbf{O}_{k-1}* \textbf{ED}_{k-1}^T* \textbf{O}_{k}^{T}
\end{equation*}
From this, we have that the matrix representation of the Hodge Laplacian $ \Delta_k (W)$ is given by
\begin{equation*}
\frac{1}{t^4}  \textbf{O}_{k}* ((k+1) \textbf{ED}_k ^{T} \textbf{ED}_{k}+ k \textbf{ED}_{k-1} \textbf{ED}_{k-1}^{T})*\textbf{O}_{k}^{T}*W
\end{equation*}
where $W$ is a  $k$-differential array. Therefore, the intrinsic information of the Hodge-Laplacian $\Delta_k$ can be captured through the matrix:
\begin{equation}
\label{generalizationdifmaps}
  \textbf{H}_{k,t}= \frac{1}{t^4} ((k+1) \textbf{ED}_k ^{T} \textbf{ED}_{k}+ k \textbf{ED}_{k-1} \textbf{ED}_{k-1}^{T})
\end{equation}
We define $\textbf{H}_{k,t}$ as the \textbf{Hodge-Laplacian} matrix of order $k$. 

Similarly to the case of the exterior derivative, where the exterior derivative matrix $\textbf{ED}_k$ is unique up to equivalence between matrices (as shown in \cref{unicidadedequi}), the Hodge-Laplacian matrix $\textbf{H}_{k,t}$ is also unique up to a similar equivalence. Specifically, for a different choice of the orthonormal basis  \( O^{'}_{j_1}(x_i), \dots, O^{'}_{j_k}(x_i) \), the corresponding \textbf{Hodge-Laplacian} matrix $\textbf{H}_{k,t}^{'}$, defined analogously to \cref{generalizationdifmaps} satisfies:
\begin{equation} \label{cambiodebasehodgelaplacian}
    \textbf{H}_{k,t}^{'}= ((\textbf{O}_{k}^{'})^{T}*\textbf{O}_{k})*\textbf{H}_{k,t}* ((\textbf{O}_{k}^{'})^{T}*\textbf{O}_{k})^{T}
\end{equation}
where $(\textbf{O}_{k}^{'})^{T}*\textbf{O}_{k}$ is an orthonormal matrix.

\subsection{Hodge Diffusion-Maps and Hodge Diffusion-Distance}

As in the definition of affinity in vector diffusion maps \cite[Page 1078]{singer2012vector}, we use the Hodge-Laplacian matrix, $\mathbf{H}_{k,t}$, to define an affinity between two points, $x_i$ and $x_j$. In this section, we describe the type of affinity that the Hodge-Laplacian captures within the dataset.

Consider the matrix $\mathbf{ED}_k$ as defined in \cref{repfinextderivative}. This matrix is constructed by incorporating terms of the form
\[
e^{\frac{-\| x_i - x_j \|^2}{2 t^2}} \det \left( O(x_i)^{T}(x_j - x_i), O_{M}^{T}(x_i) O^J(x_j) \right),
\]
where the exponential factor, $e^{\frac{-\| x_i - x_j \|^2}{2 t^2}}$, encodes the local proximity between points $x_i$ and $x_j$, reflecting their spatial relationship. The remaining factors in the determinant represent the area of the parallelogram spanned by the vectors $O(x_i)^T(x_j - x_i)$ and the matrix product $O_M^T(x_i) O^J(x_j)$, which accounts for the change of basis of the $k$ tangent vectors at both $x_i$ and $x_j$.

Thus, the block $(i,j)$ of the $\mathbf{ED}_k$ matrix quantifies both the proximity between $x_i$ and $x_j$ and the area of the parallelogram formed by these vectors. 

By construction, the $(i,j)$ block of the Hodge-Laplacian matrix measures the local connectivity between $x_i$ and $x_j$, along with the geometric structure defined by the $k$ and $k+1$-dimensional change of basis vectors at each point, in relation to other points in the dataset.

To be more specific, this affinity is defined as the squared Frobenius norm of the $\textbf{tm}$-power of the $(i,j)$-block of $\mathbf{H}_{k,t}$, i.e., 
\begin{equation*}
    \|\mathbf{H}_{k,t}^\textbf{tm}(i,j)\|_F^2= \mathrm{Tr}( \mathbf{H}^{\textbf{tm}}_{k,t}(i,j)^{T}*\mathbf{H}^{\textbf{tm}}_{k,t}(i,j))
\end{equation*}
This affinity quantifies how information from the Hodge-Laplacian matrix propagates from $x_i$ to $x_j$ along a path of length $\mathbf{tm}$. Additionally, it reflects how concentrated the information from the  $\textbf{tm}$-th power of the \textbf{Hodge-Laplacian} is when passing information from the $j$-th node to the $i$-th node. Specifically, for any  $k$-differential array $W$, the Hodge-Laplacian at the $i$-th point can be approximated as

\begin{equation*}
    \Delta^{\textbf{tm}}_{k} W(x_i)  \approx \sum_{j=1}^N  O_k(x_i) * \mathbf{H}^{\textbf{tm}}_{k,t}(i,j) * O_k^{T}(x_j) * W(x_j)
\end{equation*}
Thus, the norm $\|\mathbf{H}^{\textbf{tm}}_{k,t}(i,j)\|_F^2$ is large when the differential array $W$ at $x_j$ plays a significant role computing the $\textbf{tm}$-th power of the \textbf{Hodge-Laplacian} at the point $x_i$.

\begin{rem}
An important observation is that the affinity definition using the Frobenius norm \( \|\mathbf{H}^{\textbf{tm}}_{k,t}(i,j)\|_F^2 \) is independent of the choice of orthonormal basis for the tangent space \( T_{x_i}\mathcal{M} \)
\[
    O_1(x_j), O_2(x_j), \dots, O_d(x_j).
\]
Indeed, if \( O_1'(x_j), O_2'(x_j), \dots, O_d'(x_j) \) is another orthonormal basis, and \( \mathbf{H}_{k,t}' \) is the corresponding Hodge-Laplacian matrix, then, by \cref{cambiodebasehodgelaplacian}, the \((i,j)\)-th block matrix of the ${\textbf{tm}}$ power of \( \mathbf{H}_{k,t}\) and \( \mathbf{H}_{k,t}' \) satisfy  

\[
   ( \mathbf{H}_{k,t}')^{\textbf{tm}}(i,j) = A * \mathbf{H}^{\textbf{tm}}_{k,t}(i,j) * B^T,
\]
for some orthonormal matrices \( A \) and \( B \). This implies that
\[
    \left( (\mathbf{H}_{k,t}')^{\textbf{tm}}(i,j) \right)^T * (\mathbf{H}_{k,t}')^{\textbf{tm}}(i,j)
\]
and
\[
    \mathbf{H}^{\textbf{tm}}_{k,t}(i,j)^T * \mathbf{H}^{\textbf{tm}}_{k,t}(i,j)
\]
are similar matrices and therefore the same trace. Consequently, the Frobenius norms are equal, proving the claim.
\end{rem}

We now define Hodge Diffusion Maps. By construction, the \textbf{Hodge-Laplacian} matrix $\mathbf{H}_{k,t}$ is symmetric and non-negative definite. Thus, by the spectral theorem, it admits a complete set of eigenvectors $b_1, b_2, \dots, b_{N \binom{d}{k}}$ in $\mathbb{R}^{N \binom{d}{k}}$, with corresponding non-negative eigenvalues $\lambda_1 \geq \lambda_2 \geq \dots \geq \lambda_{N \binom{d}{k}}$. Each vector $b_j$ is considered as a block vector, where each block has size $N \times 1$ and consists of a column vector of dimension $\binom{d}{k} \times 1$. We denote the $i$-th block of $b_j$ by $b_j(i)$. Using this orthonormal eigenbasis, the $(i,j)$-th block of $\mathbf{H}^{\textbf{tm}}_{k,t}$ can be written as
\[
    \mathbf{H}^{\textbf{tm}}_{k,t}(i,j) = \sum_{l=1}^{N \binom{d}{k}} \lambda_l^{\textbf{tm}} b_l(i) \otimes b_l(j).
\]
Consequently, the affinity measure $\|\mathbf{H}^{\textbf{tm}}_{k,t}(i,j)\|_F^2$ takes the form
\begin{equation*}
%\label{ecuacionafinidad}
    \|\mathbf{H}^{\textbf{tm}}_{k,t}(i,j)\|_F^2 = \sum_{l_1, l_2 = 1}^{N \binom{d}{k}} \lambda_{l_1}^{\textbf{tm}} \lambda_{l_2}^{\textbf{tm}} \langle b_{l_1}(i), b_{l_2}(i) \rangle \, \langle b_{l_1}(j), b_{l_2}(j) \rangle.
\end{equation*}
This representation allows us to define an embedding for the dataset. For $1 \leq m \leq N \binom{d}{k}$, we define the truncated $k$-th \textbf{Hodge diffusion map} at time $\textbf{tm}$ and truncation level $m$ , denoted by $\eta^{\textbf{tm}}_{k,m}$, as the embedding that maps the dataset $X = \{ x_1, x_2, \dots, x_N \} \subseteq \mathbb{R}^n$ into $\mathbb{R}^{m \times m}$, via the square matrix:
\begin{equation}
    \label{normaembedd}
    \eta^{\textbf{tm}}_{k,m}(x_i) = \begin{bmatrix}
    \sqrt{\lambda_{l_1}^{\textbf{tm}}} \sqrt{\lambda_{l_2}^{\textbf{tm}}} \langle b_{l_1}(i), b_{l_2}(i) \rangle_{\mathbb{R}^{\binom{d}{k}}}
    \end{bmatrix}_{1 \leq l_1, l_2 \leq m}.
    \end{equation}
Here, $\langle \cdot, \cdot \rangle_{\mathbb{R}^{\binom{d}{k}}}$ denotes the standard inner product in $\mathbb{R}^{\binom{d}{k}}$.  The affinity between two points two points $x_i$ and $x_j$ can then be approximated as
\[
    \|\mathbf{H}^{\textbf{tm}}_{k,t}(i,j)\|_F^2 \approx \langle \eta^{\textbf{tm}}_{k,m}(x_i), \eta^{\textbf{tm}}_{k,m}(x_j) \rangle_F.
\]
Based on the vector diffusion distance \cite{singer2012vector}, which measures the connectivity of points using the connected Laplacian, we use the Hodge Laplacian to define the \textbf{Hodge Diffusion-Distance} $d_{\textbf{Hodge}}$ between two points \( x_i \) and \( x_j \) as:

\[
    d_{\textbf{Hodge}}^2(x_i,x_j) = \| \eta^{\textbf{tm}}_{k,m}(x_i) \|_F^2 + \| \eta^{\textbf{tm}}_{k,m}(x_j) \|_F^2 - 2 \langle \eta^{\textbf{tm}}_{k,m}(x_i), \eta^{\textbf{tm}}_{k,m}(x_j) \rangle_F.
\]
Although the embedding $\eta^{\textbf{tm}}_{k,m}$ is computed using only the first $m$ eigenvalues, an important question is when it provides a good approximation for the affinity measure $\|\mathbf{H}^{\textbf{tm}}_{k,t}(i,j)\|_F^2$. Specifically, we are interested in when the error in the approximation, given by the absolute value of the difference
\[
\|\mathbf{H}^{\textbf{tm}}_{k,t}(i,j)\|_F^2 - \langle \eta^{\textbf{tm}}_{k,m}(x_i), \eta^{\textbf{tm}}_{k,m}(x_j) \rangle_F,
\]
is small enough. In practice, this is not guaranteed, since the eigenvalues $\lambda_l^{\textbf{tm}}$ for $l > m$ could still be large, especially if $\lambda_l > 1$. To address this issue, we normalize the affinity measure, the Hodge diffusion maps, and the \textbf{Hodge Diffusion-Distance} by the factor $1 / \lambda_1^{\textbf{tm}}$. Instead of considering $\|\mathbf{H}^{\textbf{tm}}_{k,t}(i,j)\|_F^2$, $\eta^{\textbf{tm}}_{k,m}(x_i)$, and $d_{\textbf{Hodge}}^2(x_i,x_j)$, we use their normalized counterparts: 

\begin{equation}
\label{afinidadnormalizada}
    \frac{1}{\lambda_1^{2 \, \textbf{tm}}}  \|\mathbf{H}^{\textbf{tm}}_{k,t}(i,j)\|_F^2,
\end{equation}
\begin{equation}
\label{embnormalizada}
    \frac{1}{\lambda_1^{\textbf{tm}}} \eta^{\textbf{tm}}_{k,m}(x_i),
\end{equation}
and
\[
    \frac{1}{\lambda_1^{2 \, \textbf{tm}}} d_{\textbf{Hodge}}^2(x_i,x_j).
\]
With these normalizations, the error in the normalized embedding and the normalized affinity measure is bounded by:

\[
    \frac{1}{\lambda_1^{2 \, \textbf{tm}}} \left| \|\mathbf{H}^{\textbf{tm}}_{k,t}(i,j)\|_F^2 - \langle \eta^{\textbf{tm}}_{k,m}(x_i), \eta^{\textbf{tm}}_{k,m}(x_j) \rangle_F \right| 
    \le \left( \frac{\lambda_{m+1}}{\lambda_1} \right)^{\textbf{tm}} \left( (N \binom{d}{k})^2 - m^2 \right).
\]
If $m$ is chosen is chosen so that the \((m+1)\)-th eigenvalue satisfies
\[
     \frac{\lambda_{m+1}}{\lambda_1} < 1,
\]
then as $\textbf{tm}\to\infty$, the error approaches zero. This ensures that the normalized embedding, given in \cref{embnormalizada}, provides a good approximation of the normalized affinity measure in \cref{afinidadnormalizada}. 

In \cref{experimentos}, we perform several numerical experiments using these normalized quantities to demonstrate that the \textbf{Hodge Diffusion-Map} accurately approximates the affinity measure with only a small number of terms, $m$.

\section{Numerical Experiments}
\label{experimentos}
In this section, we provide a numerical validation of the proposed methodology using sample points from the two-dimensional torus \(T^2\) and the two-dimensional sphere \(S^2\). Our focus is on the normalized versions of the affinity measure, the Hodge Diffusion Maps, and the Hodge Diffusion Distance, as defined in \cref{afinidadnormalizada,embnormalizada}, respectively.

We compare the proposed methodology against several established algorithms: Vector Diffusion Maps \cite{singer2012vector}, Diffusion Maps \cite{coifman2006diffusion}, t-distributed Stochastic Neighbor Embedding (t-SNE), and Principal Component Analysis (PCA). The implementation of Hodge Diffusion Maps follows the procedure described in \cref{algoHMH2}. As a preliminary step, we apply local PCA using \cref{algoPCA} to estimate the intrinsic dimensionality of the manifold structure underlying the dataset $X$.

In our experiments, we use the parameter settings specified in \cref{tablespecification} for the Hodge Diffusion-maps. The parameter $K$ denotes the number of sample points in the neighborhood used to run \cref{algoPCA,algoHMH2}. We set $K=30$ to ensure a reasonable number of points without significantly impacting the computational cost. The threshold parameter \(\gamma\), used in the Local PCA procedure described in \cref{differentialmatrix}, is set to \(\gamma = 0.9\) to estimate the intrinsic dimension \(d\) of the manifold.

The parameter $m$ represents the number of truncated terms used to compute the embedding $\eta^{\textbf{tm}}_{k,m}$ of the Hodge diffusion maps, as defined in \cref{normaembedd}. Since $\eta^{\textbf{tm}}_{k,m}$ is a symmetric matrix, we only consider the components in the form $(i,j)$ where $1 \leq i \leq j \leq m$. We use $m=3$ to visualize the results based on the first three terms.

The parameter $\textbf{tm}$ indicates the number of paths used to measure the connectivity between two points using the Hodge Laplacian Matrix to the power  $\textbf{tm}$. In our experiments, we set $\textbf{tm}=1$, though similar results were obtained with different values of $\textbf{tm}$. These results suggest some stable behavior on the parameter $\textbf{tm}$ and should be further investigated in future work.

For a dataset $X=\{ x_i \}_{i=1}^N$, the parameter $t$ is the diffusion scaling factor. It is set as the average of the minimum distances between each point $x_i$ and all other points $x_j$ in the dataset. The choice of $t$ is based on the need to select a small enough value to capture the data's structure, but not too small, as this could cause the term $e^{-\|x_i-x_j\|^2/2t^2}$ to vanish, losing important topological information.
\begin{table}[!ht]
\begin{center}
\begin{tabular}{|c | c | l |} 
 \hline
Parameter & Value & Description  \\ [0.5ex] 
 \hline
 $K$ & 30 & Number of points in the neighborhood \\ [0.5ex] 
  $\gamma$ & 0.9 & Threshold parameter to estimate $d$ \\ [0.5ex] 
 $m$ & 3 & Truncation level \\ [0.5ex] 
  $\textbf{tm}$ & 1 & Number of paths used to measure the \\ 
    & & connectivity between two points  \\[0.5ex] 
 $t$ & $\textbf{mean}_{i} \min_{j \neq i} \|x_i-x_j \| $ & Diffusion scaling parameter \\ 
 \hline
\end{tabular}
\caption{Parameters specification for the Hodge diffusion-Map}
\label{tablespecification}
\end{center}
\end{table}

Additionally, by applying the Cauchy–Schwarz inequality, we observe that the \((l_1, l_2)\)  component of both normalized Hodge diffusion maps and vector diffusion maps at a point \(x_i\) s dominated by the square root of the diagonal components \((l_1, l_1)\) and \((l_2, l_2)\). This suggests that the diagonal components \((l_k, l_k)\) encode information about the intensity of the diffusion of the embedding elements $\eta^{\textbf{tm}}_{k,m}$. 

In our numerical experiments, we plot the diagonal embedding of the normalized Hodge diffusion maps. Specifically, we plot the map:
\[
x_i \to \frac{1}{\lambda_1} \left( \eta^{\textbf{1}}_{k,3}(x_i)(l,l)  \right)_{1 \le l \le m}
\]
We refer to this representation as the diagonal of the normalized Hodge diffusion maps.

In the following experiments, we examine two-dimensional manifolds, namely the torus \( T^2 \) and the sphere \( S^2 \), each sampled with 2500 points distributed across them as described below. 

For the torus \( T^2 \), we use the parametrization:
\[
\Omega(u, v) = \left[ (2 + \cos(2 \pi v)) \cos( 2 \pi u), (2 + \cos( 2 \pi v)) \sin(2 \pi u), \sin(2 \pi v) \right]
\]
where \( -\frac{1}{2} \leq u, v \leq \frac{1}{2} \). To construct  the dataset, we define $50$ evenly spaced sample points \( u_1, u_2, \dots, u_{50} \) within the interval \( \left[-\frac{1}{2}, \frac{1}{2}\right) \) using:
\[
u_i = \frac{i - 1}{50} - \frac{1}{2} \quad \text{for} \quad 1 \leq i \leq 50.
\]
Using this grid, the dataset \( X \) is then:
\[
X = \left\{ \Omega(u_i, u_j) \right\}_{1 \leq i, j \leq 50},
\]
resulting in 2500 points distributed over \( T^2 \).

For the sphere \( S^2 \), we use the following parametrization:
\[
\Omega(u, v) = \left[ \cos(2 \pi u) \sin(\pi v), \sin(2 \pi u) \sin(\pi v), \cos(\pi v) \right]
\]
where \( 0 \leq u, v \leq 1 \). To create the dataset, we define $50$ evenly spaced sample points \( u_1, u_2, \dots, u_{50} \) within the interval \( \left[0, 1\right) \), given by:
\[
u_i = \frac{i - 1}{50} \quad \text{for} \quad 1 \leq i \leq 50.
\]
The resulting dataset \( X \) is defined as:
\[
X = \left\{ \Omega(u_i, u_j) \right\}_{1 \leq i, j \leq 50},
\]
yielding in 2500 points distributed over \( S^2 \).

In both parametrization systems, \( u_i \) and \( u_j \) correspond to the first and second coordinates, respectively. In \cref{figuratoro,figuraesfera}, we visualize the datasets sampled over the torus and sphere, respectively.  The colorbars indicate the ordering of the sample points, providing a reference for their distribution across the surfaces. 

\begin{figure}[htp]
\centering
\includegraphics[width=0.8\textwidth]{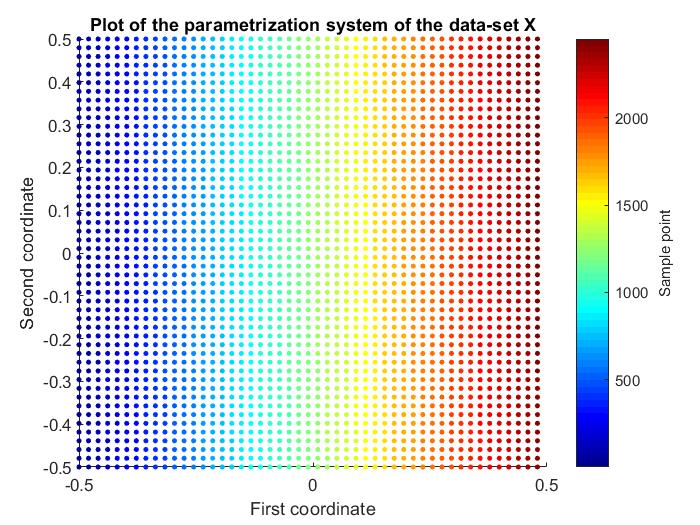}
\includegraphics[width=0.9\textwidth]{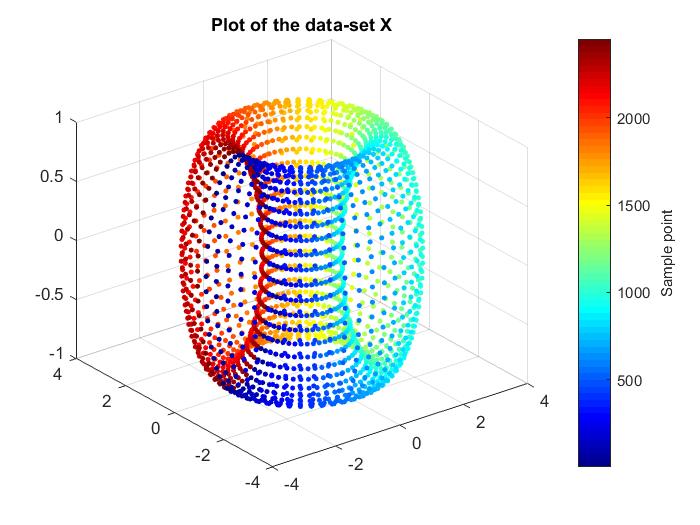}
\caption{Top: The first and second coordinates of the parametrization system for the torus. Bottom: The dataset \( X \) plotted on the torus \( T^2 \), with the colorbar indicating the order of the sample points.}
\label{figuratoro}
\end{figure}

\begin{figure}[htp]
\centering
\includegraphics[width=0.8\textwidth]{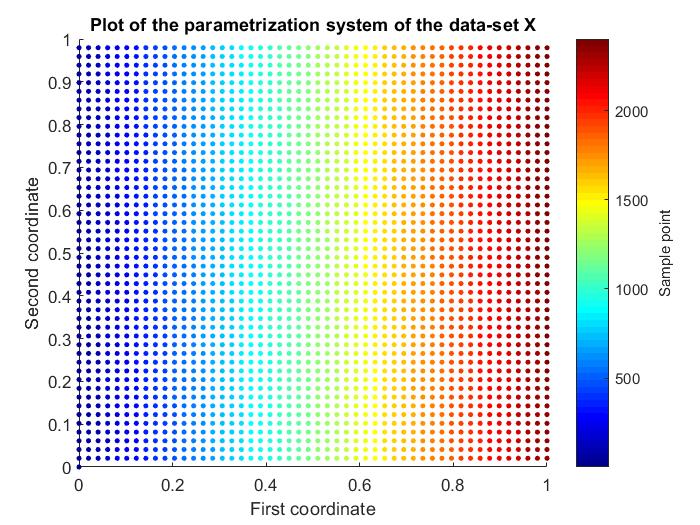}
\includegraphics[width=0.9\textwidth]{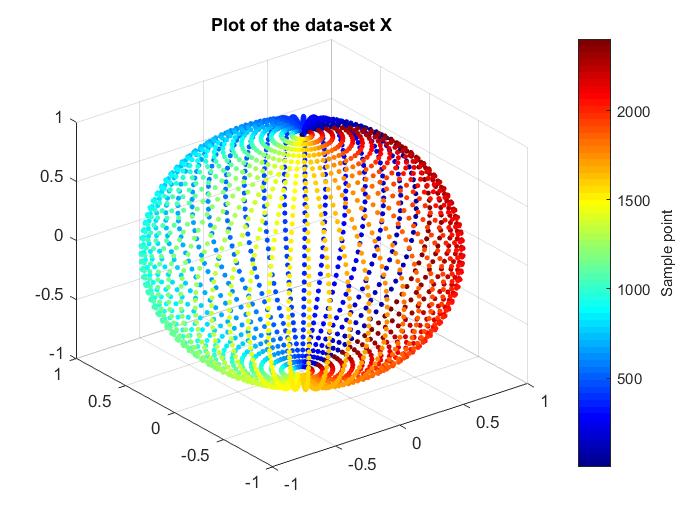}
\caption{Top: The first and second coordinates of the parametrization system for the sphere. Bottom: The dataset \( X \) plotted on the sphere \( S^2 \), with the colorbar indicating the order of the sample points.}
\label{figuraesfera}
\end{figure}

Although the dataset \( X \) consists of points sampled from each manifold, the number of points may not be large enough to fully capture the entire manifold. As a result, the dataset \( X \) could potentially represent a submanifold within the manifold or a totally different manifold from the one theoretically assumed.

The goal of this experiment is to explore how the Hodge Diffusion Maps method can be used to extract topological information from the sample dataset in \( X \). Using the local PCA algorithm (\cref{algoPCA}), we estimate the intrinsic dimension of both the torus and sphere datasets to be \( d = 2 \). Consequently, we can apply the Hodge diffusion-maps embedding up to the second order, that is, for $k \in \{1,2\}.$

In the next section, we present the results and analysis for each manifold. All experiments were performed using \textsc{Matlab} software on a laptop equipped with an Intel Core i5-1235U 1.30 GHz processor and 8 GB of RAM. The algorithms used in our implementation are available in the GitHub repository \cite{alvarorepo}.

%%%%%End parte nueva%%%%%%
\subsection{Results over two-dimensional torus}

The first-order normalized Hodge diffusion maps embedding (\( k=1 \)) is shown in \cref{figuratoro1}, while the second-order Hodge diffusion maps embedding (\( k=2 \)) is presented in \cref{figuratoro2}. Additionally, in \cref{vdmtoro}, we show the vector diffusion maps embedding. The computational time for running the Hodge diffusion maps was 98.67 seconds for the first order and 14.69 seconds for the second order.
  
The Hodge diffusion map embeddings, for both first and second orders, reveal two distinct regions with different features. One region is concentrated around points where \( u_2 \) is close to 0, while the other lies outside this area. Within each region, the values of the $(i,j)$ component exhibit similar characteristics, as indicated by distinct color patterns unique to each region. This shows that the Hodge diffusion map successfully identifies two regions with different structural characteristics.

Similarly, the vector diffusion map identifies two regions: one near points where \( u_2 \) is approximately -0.5 or 0.5, and another outside this area. Both algorithms, thus detect a partition of the dataset into two regions. While the specific regions identified by each method are not identical, they are closely related through the Weitzenböck identity, which connects the Hodge Laplacian with the Connection Laplacian operator.

\begin{figure}[htp]
\centering
\includegraphics[width=0.49\textwidth]{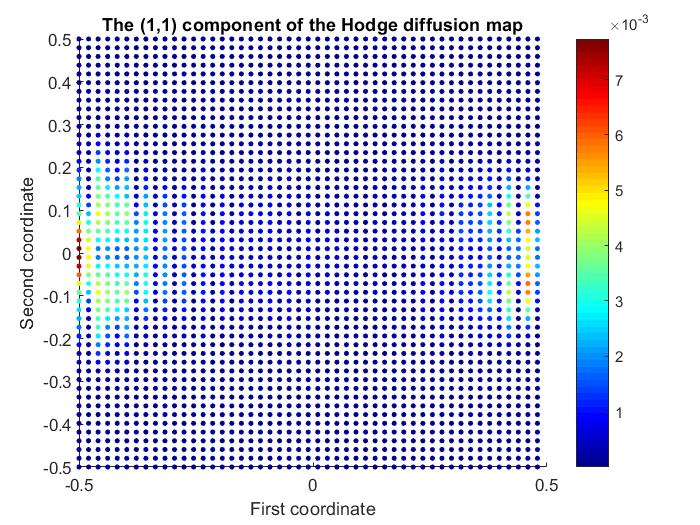}
\includegraphics[width=0.49\textwidth]{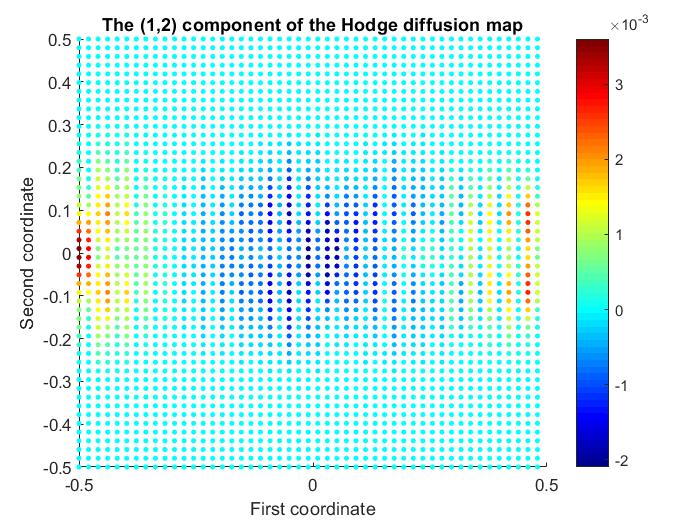}
\includegraphics[width=0.49\textwidth]{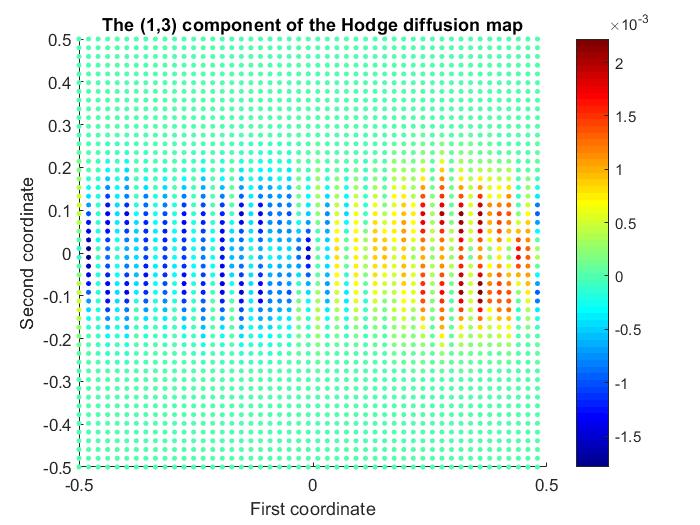}
\includegraphics[width=0.49\textwidth]{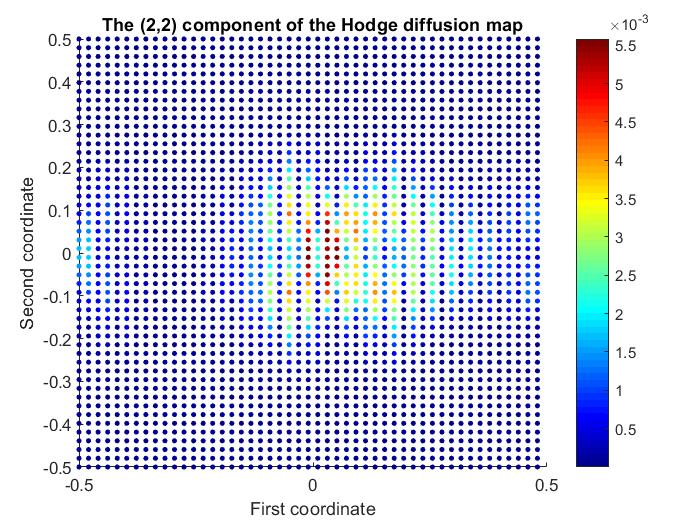}
\includegraphics[width=0.49\textwidth]{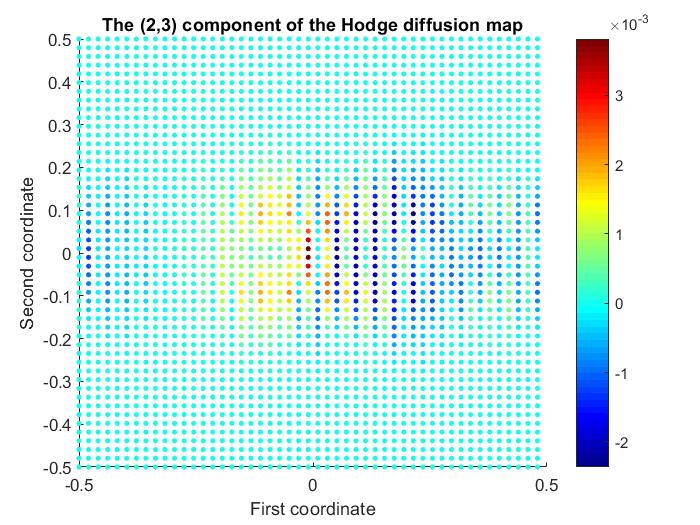}
\includegraphics[width=0.49\textwidth]{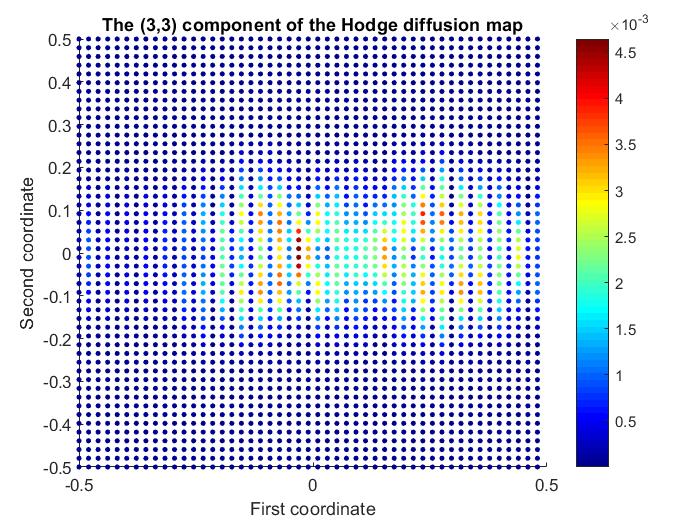}
\caption{Plot of the $(c_1, c_2)$ components, where \(1 \leq c_1 \leq c_2 \leq 3\), of the first order normalized Hodge Diffusion Maps $\eta^{\textbf{1}}_{1,3}$ as given in Eq. \eqref{embnormalizada}.}
\label{figuratoro1}
\end{figure}

\begin{figure}[htp]
\centering
\includegraphics[width=0.49\textwidth]{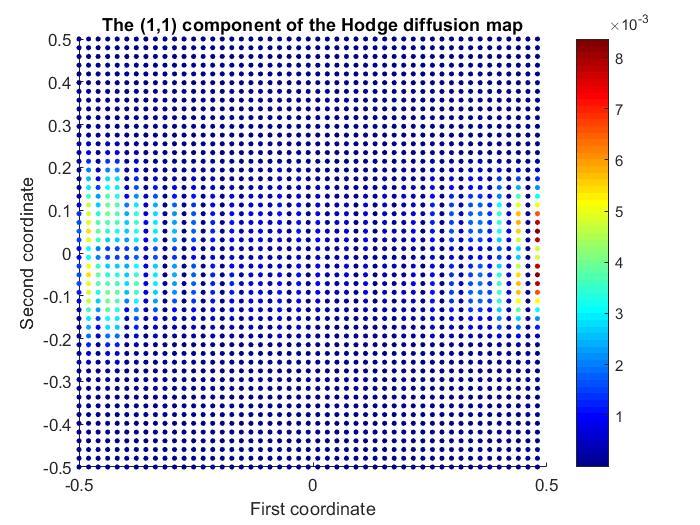}
\includegraphics[width=0.49\textwidth]{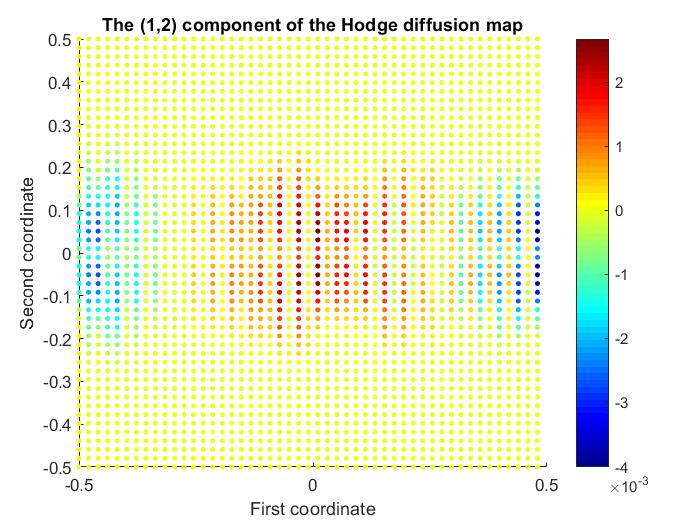}
\includegraphics[width=0.49\textwidth]{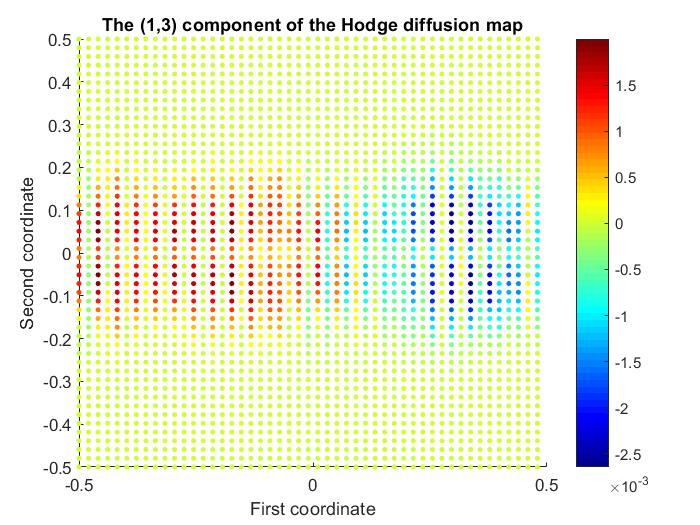}
\includegraphics[width=0.49\textwidth]{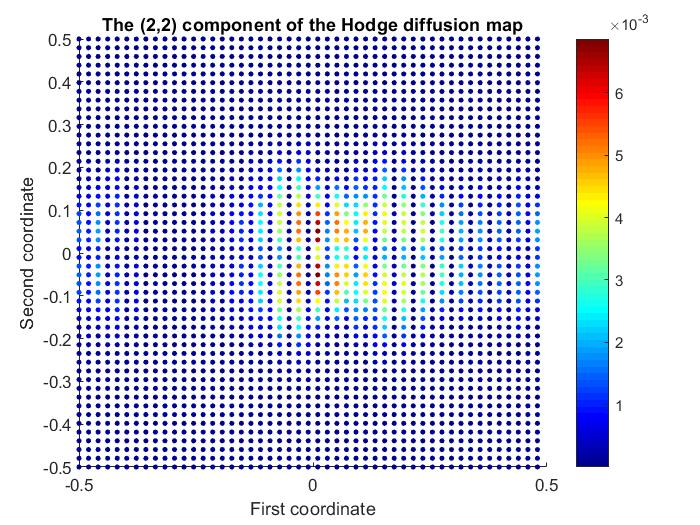}
\includegraphics[width=0.49\textwidth]{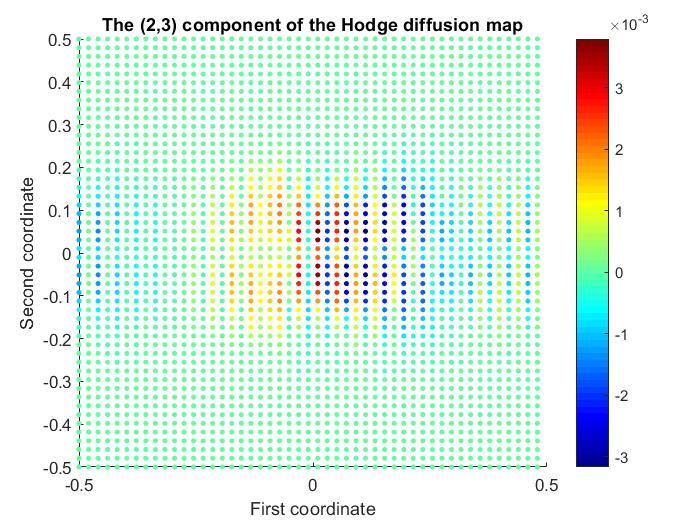}
\includegraphics[width=0.49\textwidth]{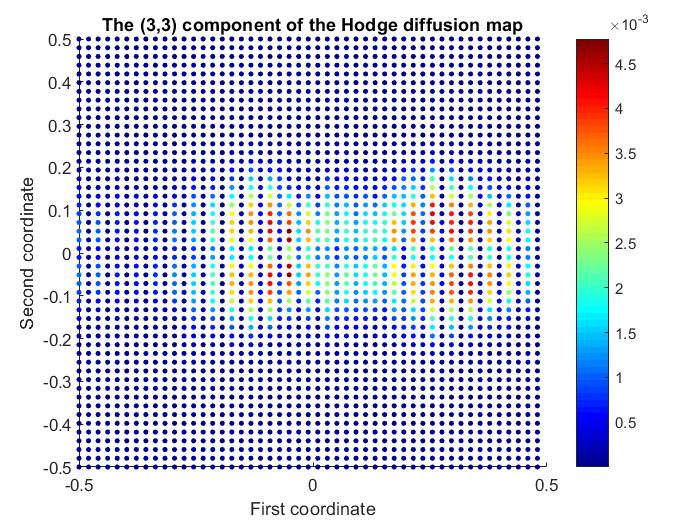}
\caption{Plot of the $(c_1, c_2)$ components, where \(1 \leq c_1 \leq c_2 \leq 3\), of the second order normalized Hodge Diffusion Maps $\eta^{\textbf{1}}_{2,3}$ as given in Eq. \eqref{embnormalizada}}
\label{figuratoro2}
\end{figure}

\begin{figure}[htp]
\centering
\includegraphics[width=0.49\textwidth]{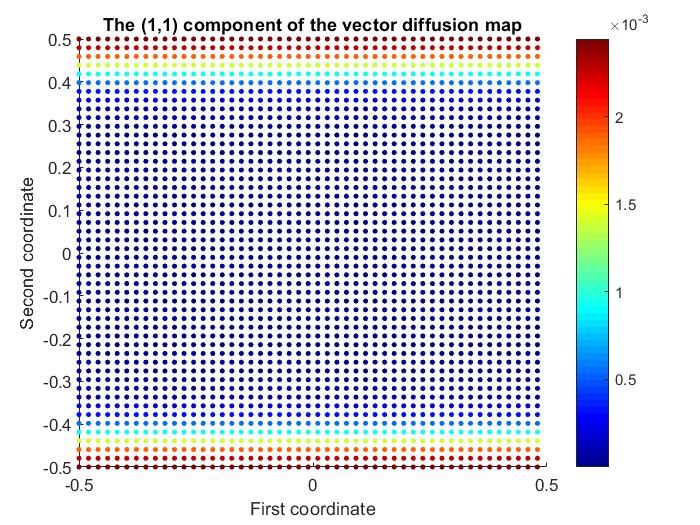}
\includegraphics[width=0.49\textwidth]{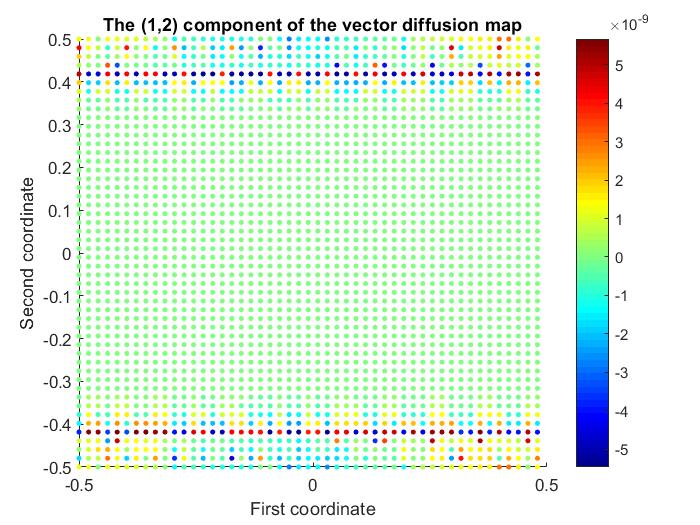}
\includegraphics[width=0.49\textwidth]{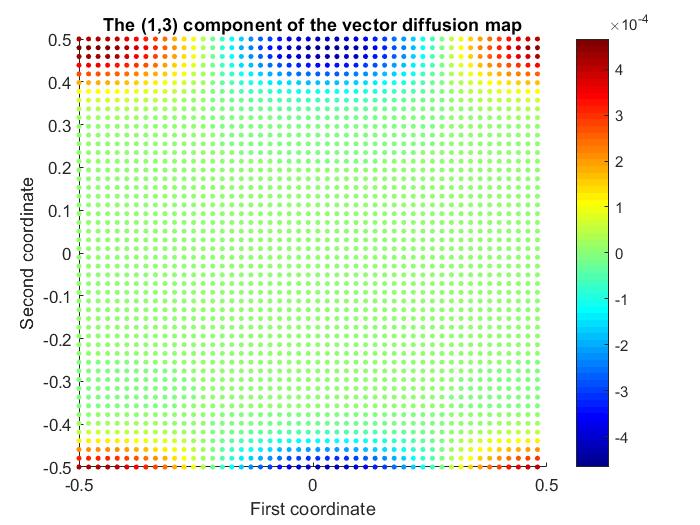}
\includegraphics[width=0.49\textwidth]{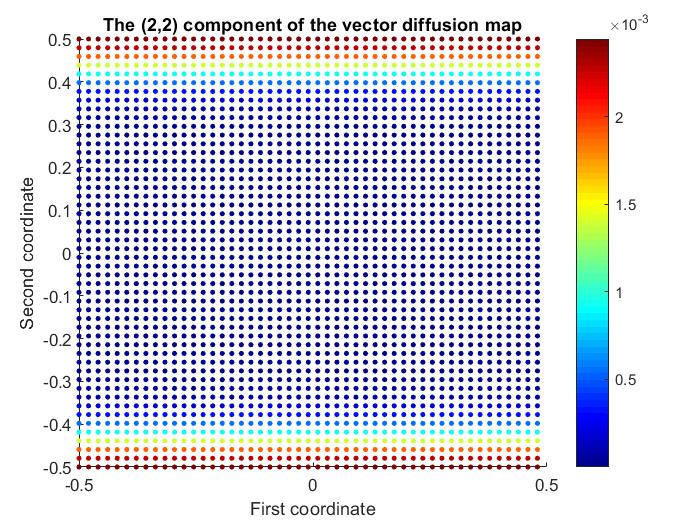}
\includegraphics[width=0.49\textwidth]{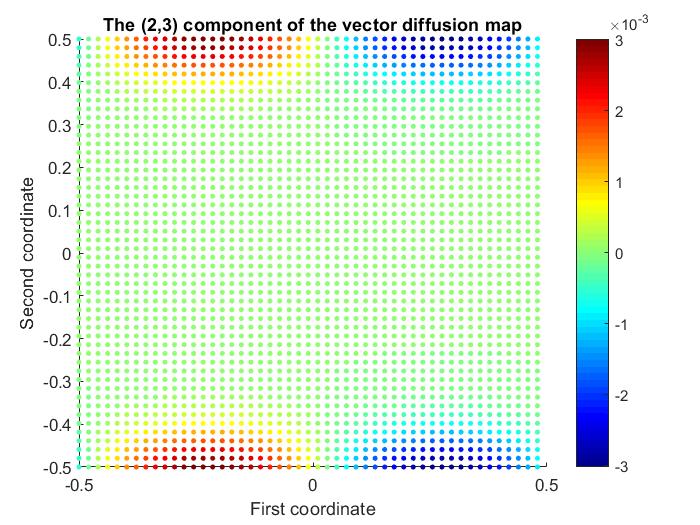}
\includegraphics[width=0.49\textwidth]{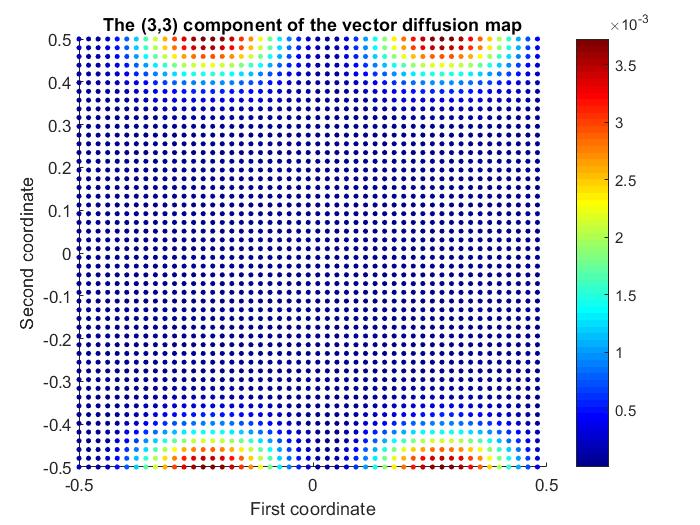}
\caption{Plot of the $(c_1, c_2)$ components, where \(1 \leq c_1 \leq c_2 \leq 3\), of vector diffusion maps.}
\label{vdmtoro}
\end{figure}

Additionally, in \cref{embedingdiagtoro}, we plot the diagonal of the normalized Hodge and vector diffusion maps. The first and second rows show the diagonals of the normalized Hodge diffusion maps for the first ($k=1$) and second ($k=2$) orders, respectively, while the third row shows the vector diffusion maps. The left column contains the first two components, and the right column contains the first three components of the respective diagonals.

We compare the proposed methodology with the t-SNE, PCA, and diffusion map algorithms applied to the dataset \( X \), as shown in \cref{otrostoro}. The colorbar and dataset organization are the same as in Figure \cref{figuratoro}.

The Hodge diffusion map, for both first and second order (\( k=1 \) and \( k=2 \)), map the vertical sections of the torus, corresponding where the first coordinate \( u_i \) is constant and assigned the same color, to approximations of straight lines in the two and three dimensional space. In contrast, vector diffusion maps represent the entire data set as several parallel straight lines, without distinguishing the vertical sections. The t-SNE algorithm transforms these vertical sections into nonlinear curves, while PCA algorith projects the torus onto a two-dimensional plane, where the vertical sections collapse into overlapping ellipses. Diffusion Maps, on the other hand, arrange the vertical sections into lines forming a circular pattern.

The results show that both diffusion maps and Hodge diffusion maps are capable of identifying and classifying vertical sections by mapping them to approximate straight lines in two- or three-dimensional spaces. This suggests that these algorithms extract topological features by mapping points from the same vertical section onto a straight line in the embedded space. As a result, linear classifiers can be used as a postprocessing step to efficiently classify data points based on their topological features.

\begin{figure}[htp]
\centering
\includegraphics[width=0.49\textwidth]{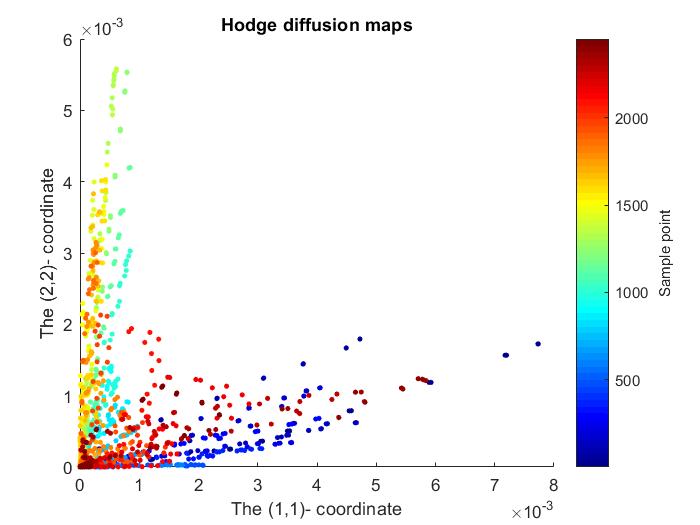}
\includegraphics[width=0.49\textwidth]{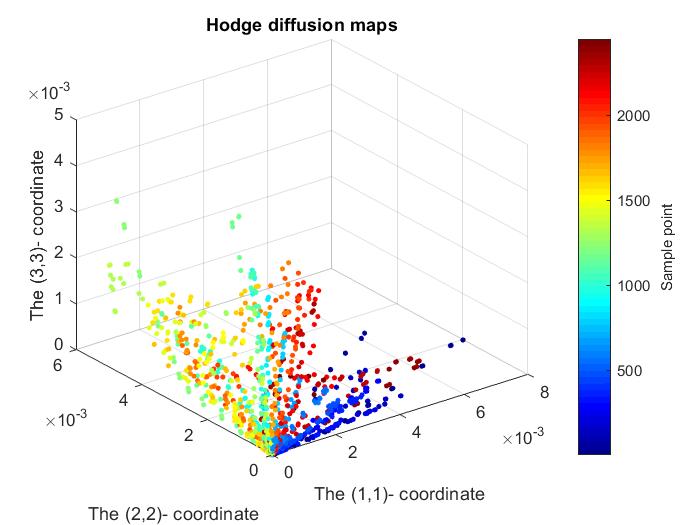}
\includegraphics[width=0.49\textwidth]{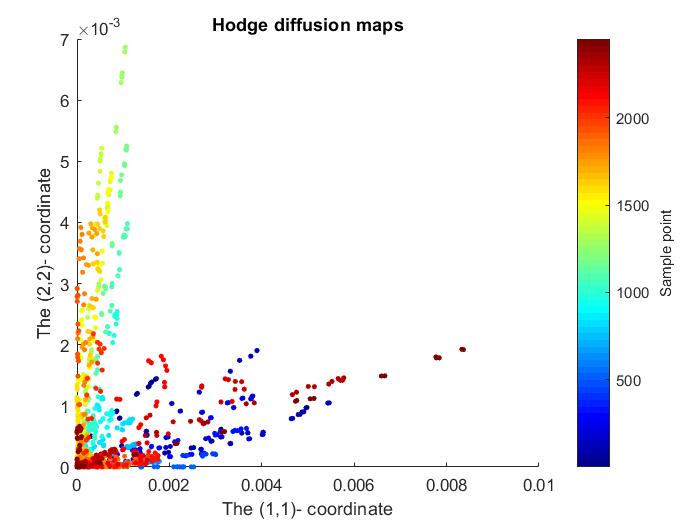}
\includegraphics[width=0.49\textwidth]{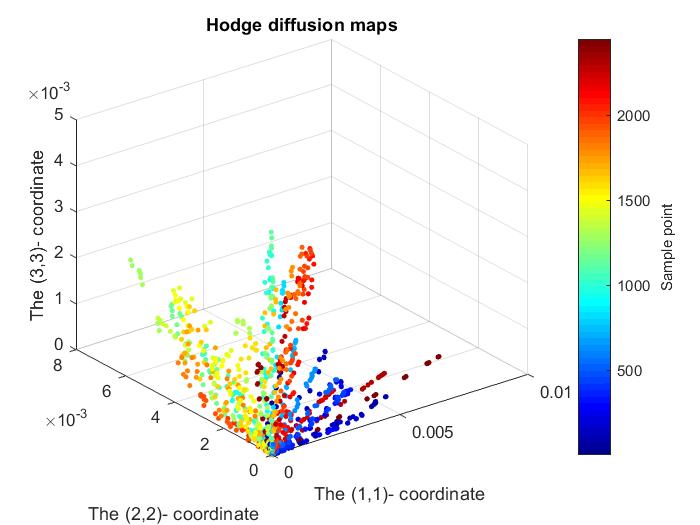}
\includegraphics[width=0.49\textwidth]{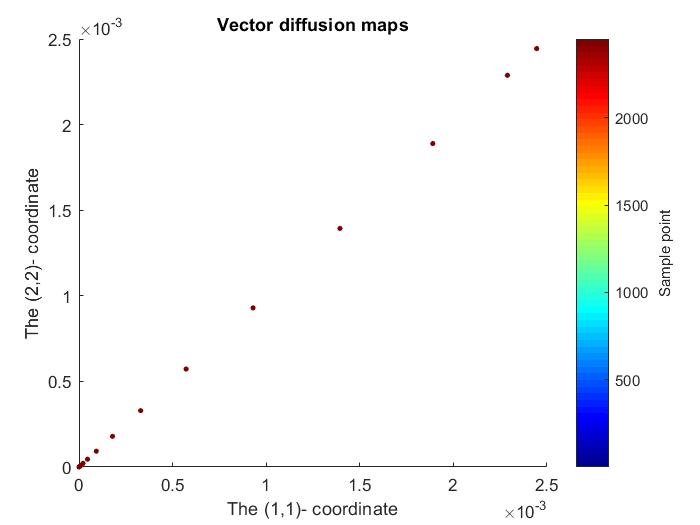}
\includegraphics[width=0.49\textwidth]{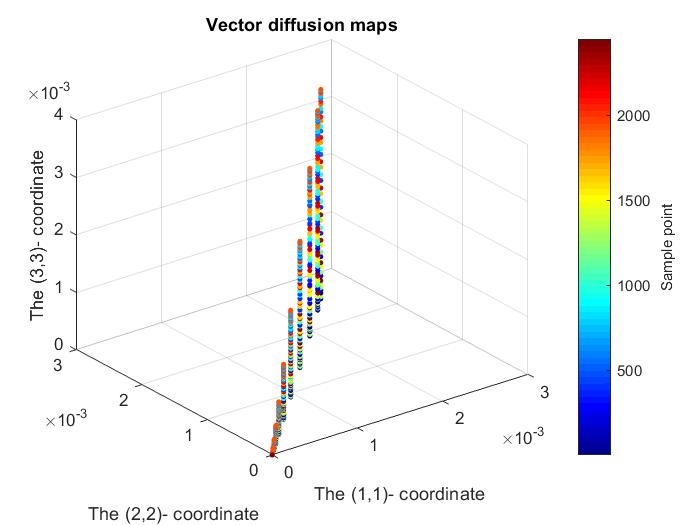}
\caption{Plot of the diagonal coordinates of the Hodge diffusion maps and vector diffusion maps. The first row shows the Hodge diffusion maps of the first order ($k=1$). The second row shows the second order ($k=2$) Hodge diffusion maps, and the third row shows the vector diffusion map. In the left column, we plot the first two diagonal coordinates, $(1,1)$ and $(2,2)$, and in the right column, we plot the first three diagonal coordinates, $(1,1)$, $(2,2)$, and $(3,3)$. The colorbar indicates the order of the points, matching the colorbar in Figure \ref{figuratoro}.}
\label{embedingdiagtoro}
\end{figure}

\begin{figure}[htp]
\centering
\includegraphics[width=0.49\textwidth]{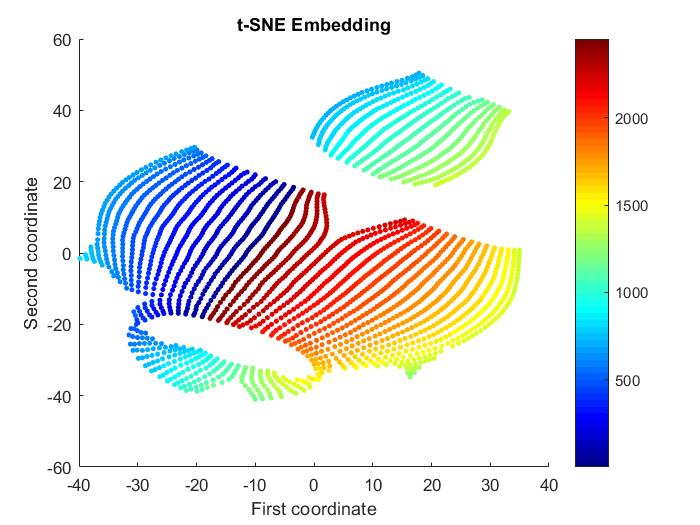}
\includegraphics[width=0.49\textwidth]{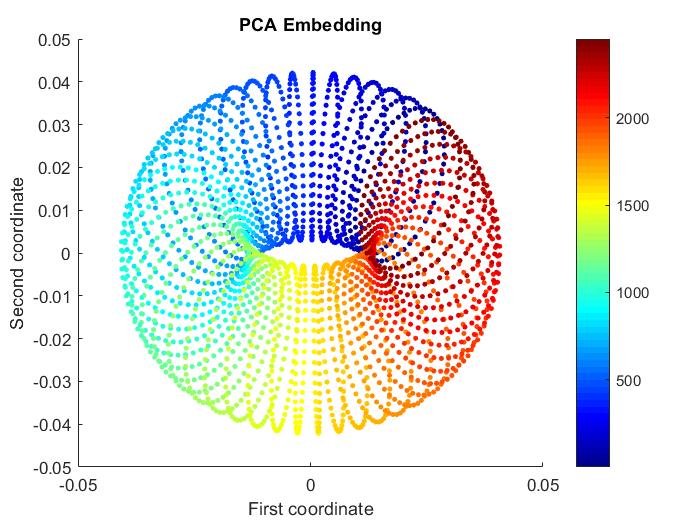}
\includegraphics[width=0.49\textwidth]{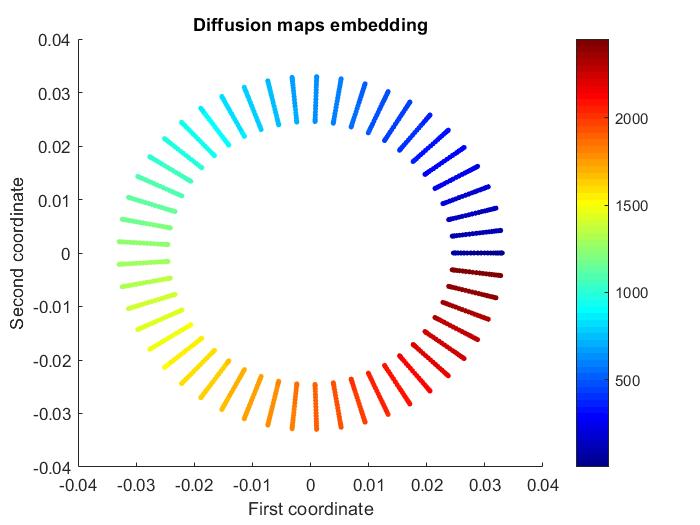}
\caption{Plot of the t-SNE, PCA, and Diffusion Maps algorithms applied to the dataset sampled on the torus.}
\label{otrostoro}
\end{figure}

\newpage

\subsection{Results over the two-dimensional sphere}

The first-order normalized Hodge diffusion maps embedding (\( k=1 \)) is shown in \cref{figuraesfera1}, while the second-order embedding (\( k=2 \)) is presented in \cref{figuraesfera2}. Additionally, the vector diffusion maps embedding is displayed in \cref{vdmesfera}. The computation times were 87.98 seconds for the first-order Hodge diffusion maps and 13.86 seconds for the second-order.

The results show that both the first- and second-order Hodge embeddings reveal two regions with distinct characteristics: one where the second coordinate \(u_2\) is close to 0.5, and another outside this range. Within each region, the values of the \( (i,j) \) components exhibit similar patterns, as reflected in the distinct color patterns unique to each region. In contrast, the vector diffusion maps also identify two regions—one where \(u_2\) is near 0 and another outside this region. Similar to the Torus case, both algorithms detect two separate regions, which are linked by the Weitzenböck identity, connecting the Hodge Laplacian and the Connection Laplacian operator.

\begin{figure}[htp]
\centering
\includegraphics[width=0.49\textwidth]{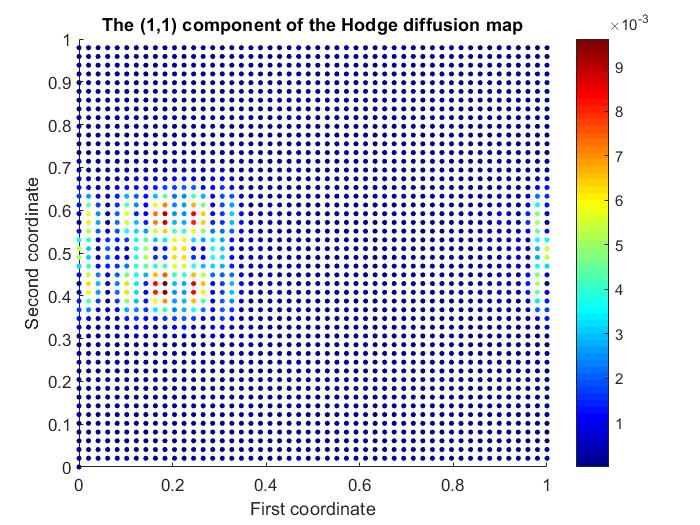}
\includegraphics[width=0.49\textwidth]{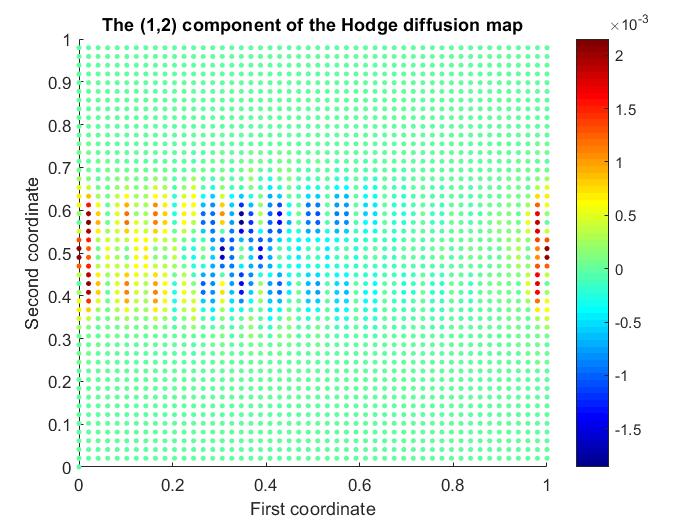}
\includegraphics[width=0.49\textwidth]{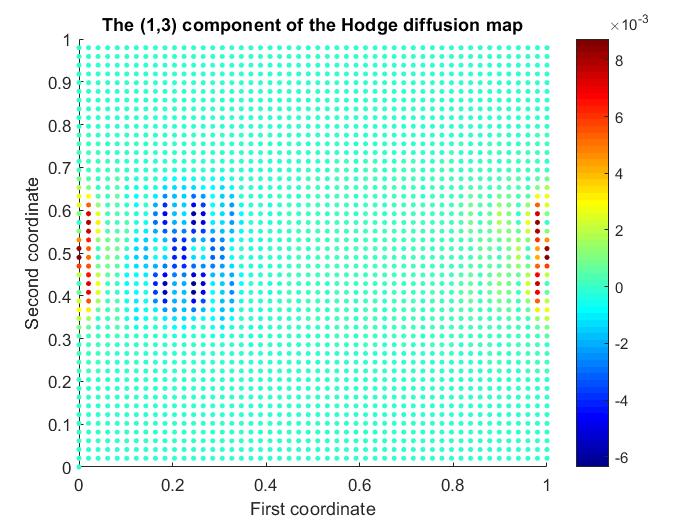}
\includegraphics[width=0.49\textwidth]{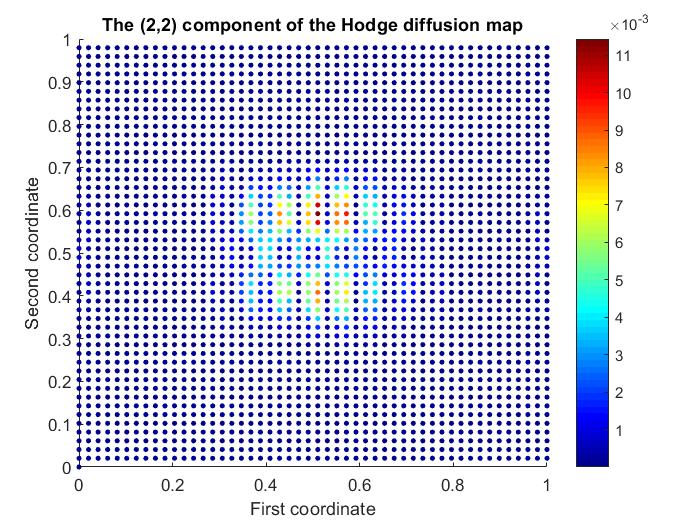}
\includegraphics[width=0.49\textwidth]{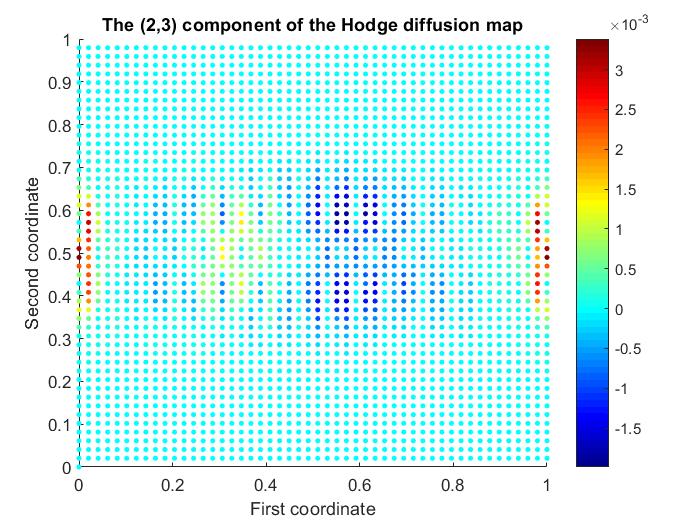}
\includegraphics[width=0.49\textwidth]{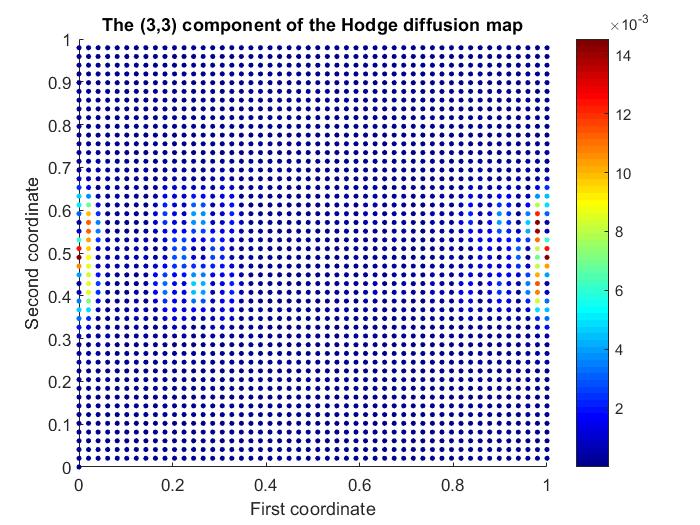}
\caption{Plot of the $(c_1, c_2)$ components, where \(1 \leq c_1 \leq c_2 \leq 3\), of the first order normalized Hodge Diffusion Maps $\eta^{\textbf{1}}_{1,3}$ as given in Eq. \eqref{embnormalizada}, for sampled points on the sphere.}
\label{figuraesfera1}
\end{figure}

\begin{figure}[htp]
\centering
\includegraphics[width=0.49\textwidth]{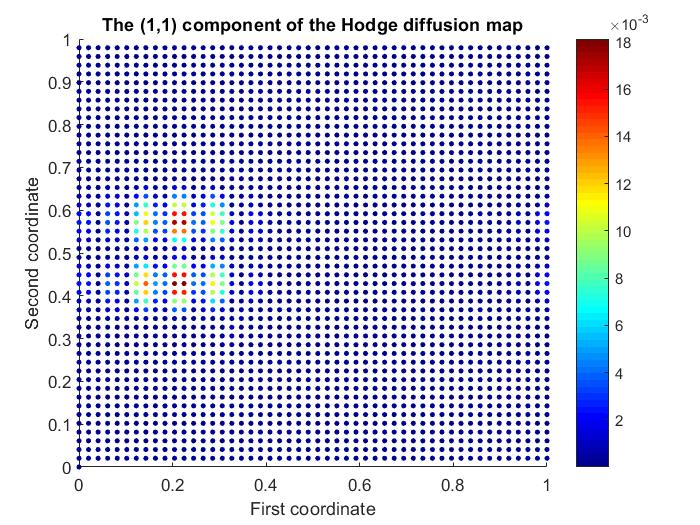}
\includegraphics[width=0.49\textwidth]{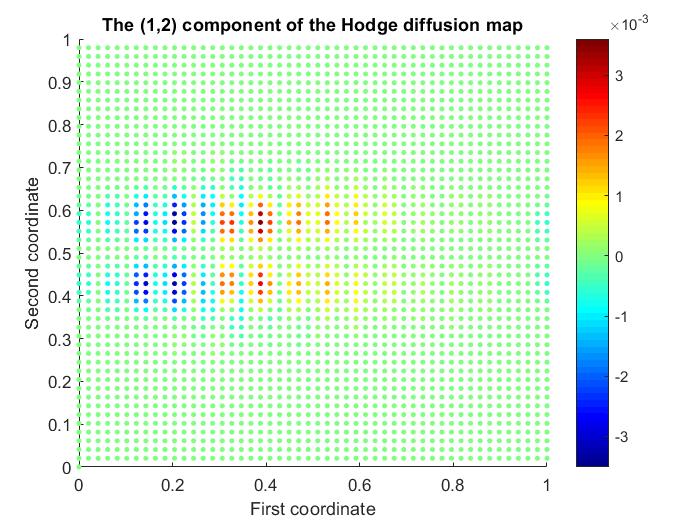}
\includegraphics[width=0.49\textwidth]{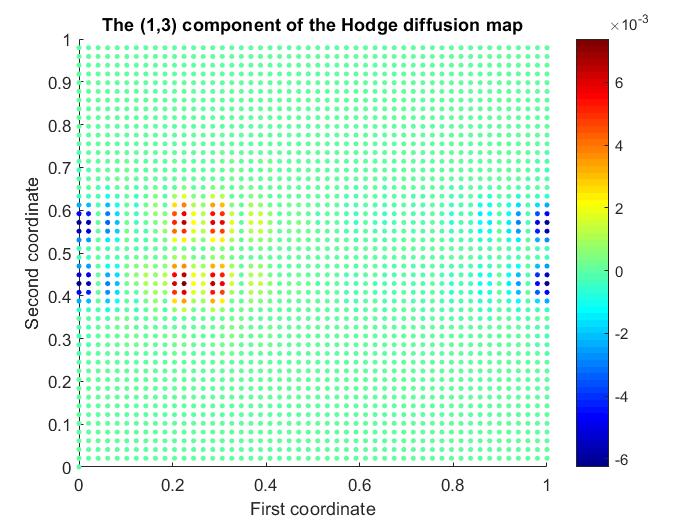}
\includegraphics[width=0.49\textwidth]{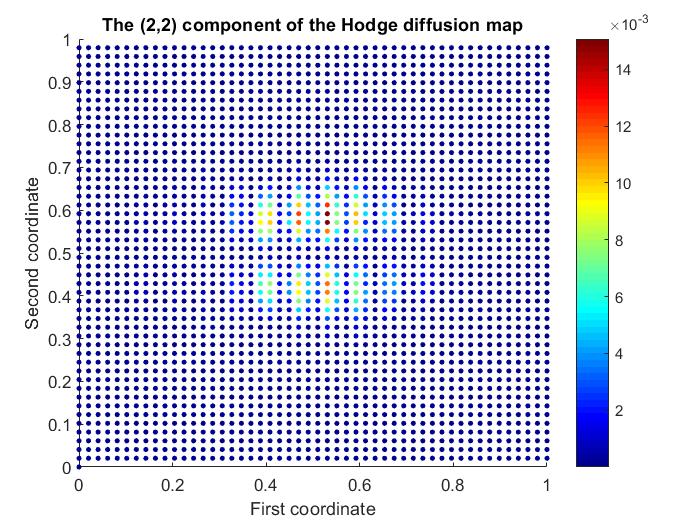}
\includegraphics[width=0.49\textwidth]{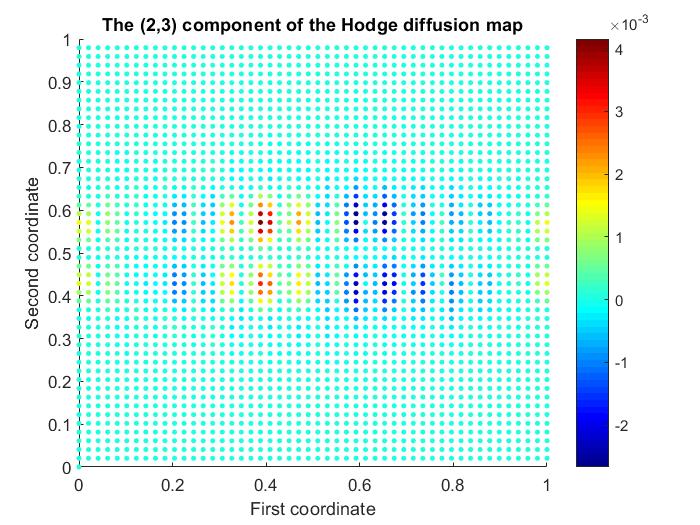}
\includegraphics[width=0.49\textwidth]{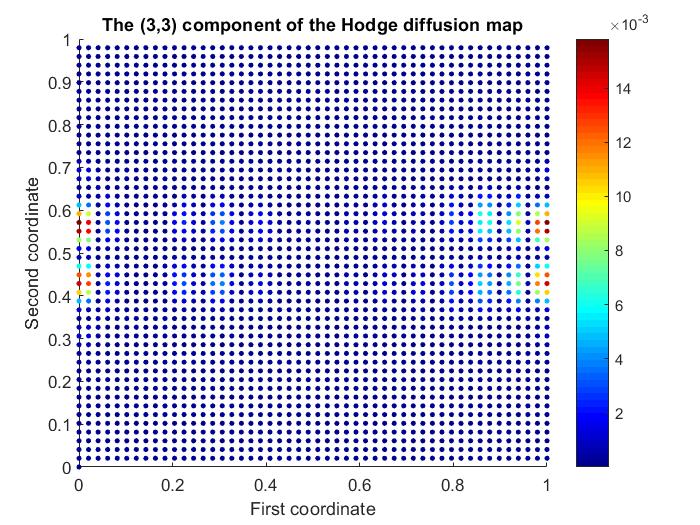}
\caption{Plot of the $(c_1, c_2)$ components, where \(1 \leq c_1 \leq c_2 \leq 3\), of the second order normalized Hodge Diffusion Maps $\eta^{\textbf{1}}_{2,3}$ as given in Eq. \eqref{embnormalizada}, for sampled points on the sphere.}
\label{figuraesfera2}
\end{figure}

\begin{figure}[htp]
\centering
\includegraphics[width=0.49\textwidth]{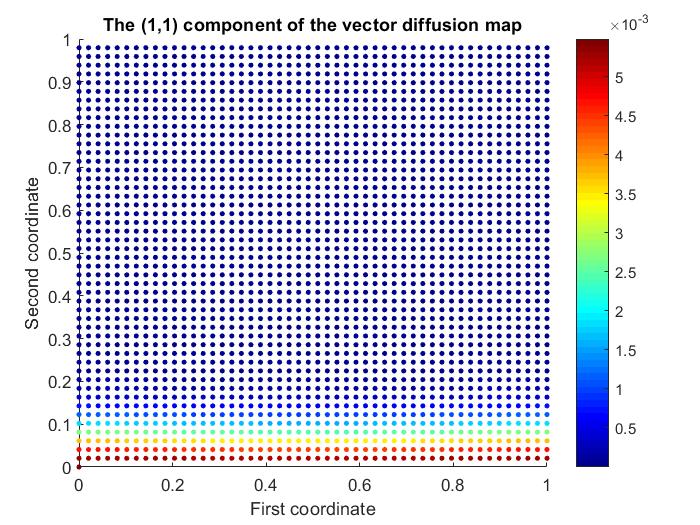}
\includegraphics[width=0.49\textwidth]{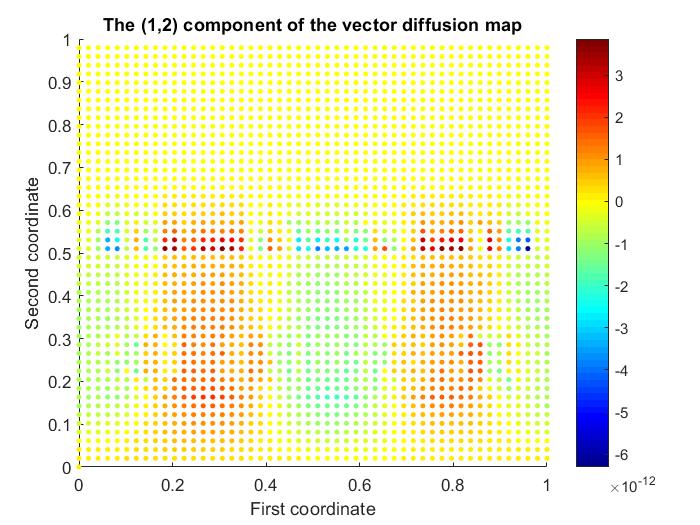}
\includegraphics[width=0.49\textwidth]{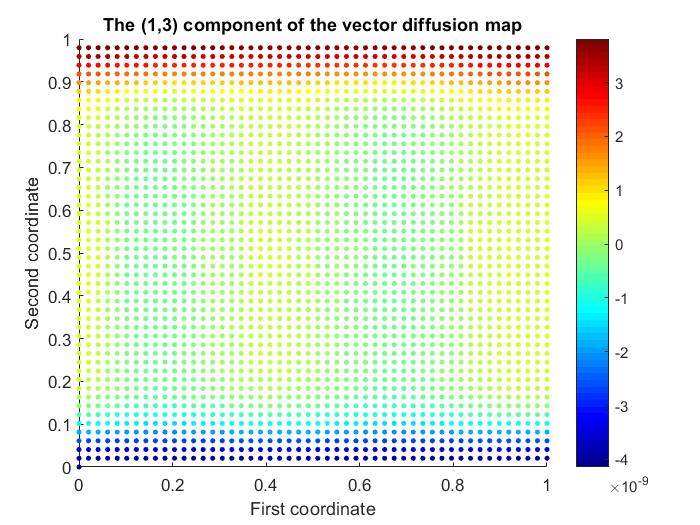}
\includegraphics[width=0.49\textwidth]{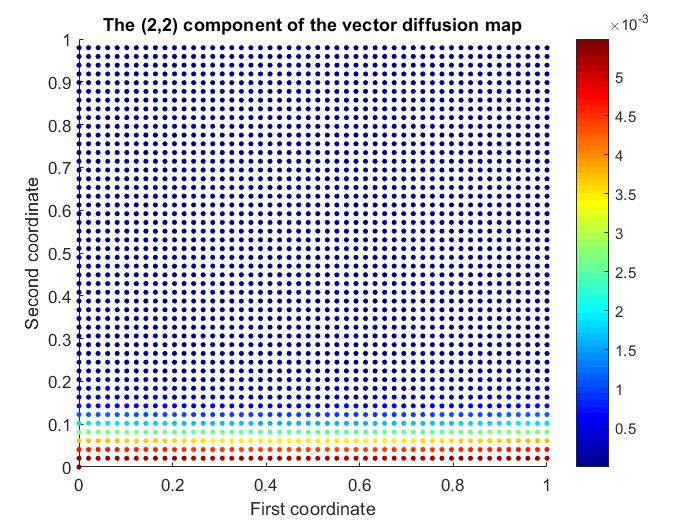}
\includegraphics[width=0.49\textwidth]{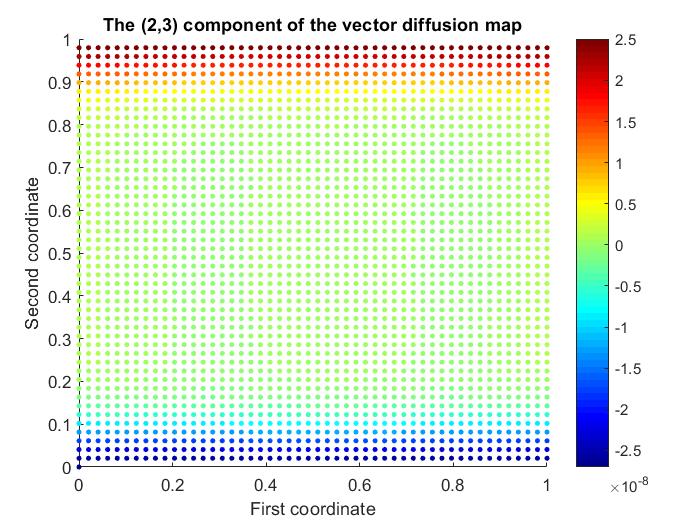}
\includegraphics[width=0.49\textwidth]{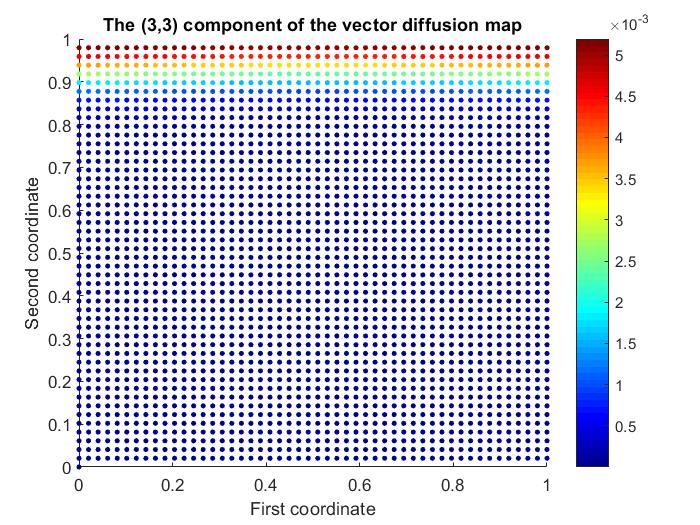}
\caption{Plot of the $(c_1, c_2)$ components, where \(1 \leq c_1 \leq c_2 \leq 3\), of vector diffusion maps for sampled points on the sphere.}
\label{vdmesfera}
\end{figure}

Additionally, in \cref{embedingdiagesfera}, we display the diagonals of the normalized Hodge and vector diffusion maps embeddings. The first and second rows correspond to the first-order ($k=1$) and second-order ($k=2$) normalized Hodge diffusion maps, respectively, while the third row shows the vector diffusion maps. The left column contains the first two components, and the right column contains the first three components of the respective diagonals.

To evaluate the proposed methodology, we compare it with t-SNE, PCA, and diffusion maps algorithms applied to the dataset $X$  as shown in \cref{otrosesfera}. The colorbar and dataset organization follow the same convention as in \cref{figuraesfera}. As illustrated in \cref{embedingdiagesfera}, the two- and three-dimensional embeddings produced by the first- and second-order Hodge diffusion maps (\(k=1\) and \(k=2\)), respectively, map the vertical sections of the dataset \(X\), corresponding to points where the first coordinate \(u_i\) is constant and assigned the same color, into distinct regions. These regions can then be separated by linear classifiers, providing a method to divide the dataset based on its vertical sections. Thus, the proposed methodology provides a useful toolbox for classifying points in the dataset based on topological patterns.

In contrast, vector diffusion maps embed the dataset into two straight lines, failing to differentiate between distinct vertical sections. Similarly, both the t-SNE and PCA algorithms transform the vertical sections into nonlinear curves, making it difficult to apply linear classifiers to separate the data based on these sections. Additionally, the diffusion map algorithm struggles to differentiate between these vertical sections. 

Among all the algorithms tested, Hodge diffusion maps is the only method that enables the use of linear classifiers in the embeddings to categorize the dataset based on points with similar topological structures defined by the vertical sections.

\begin{figure}[htp]
\centering
\includegraphics[width=0.49\textwidth]{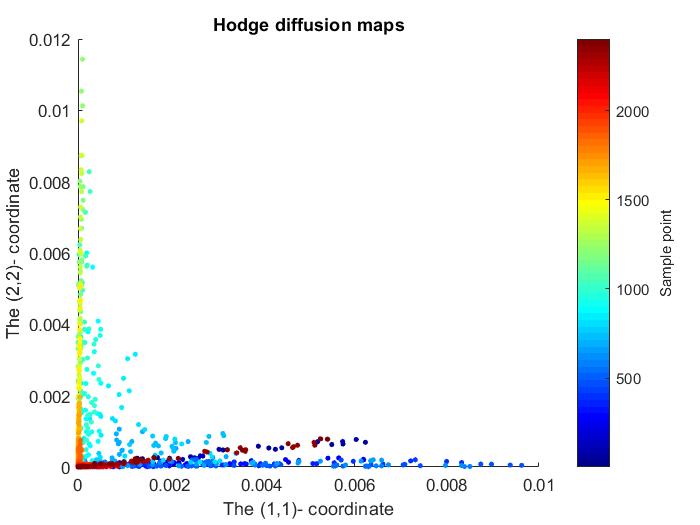}
\includegraphics[width=0.49\textwidth]{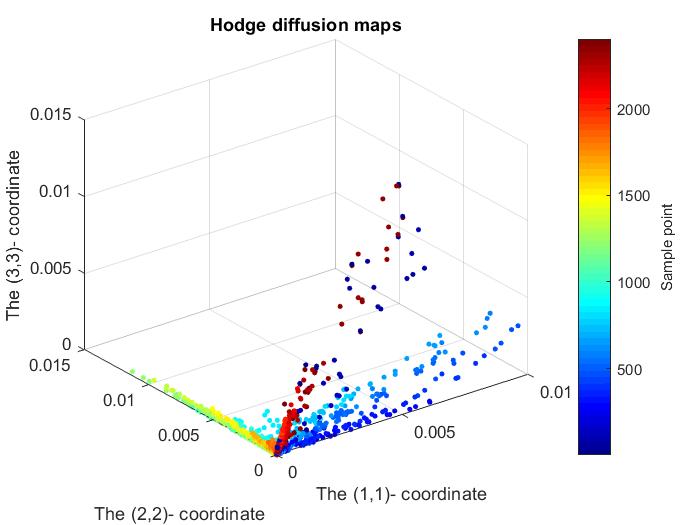}
\includegraphics[width=0.49\textwidth]{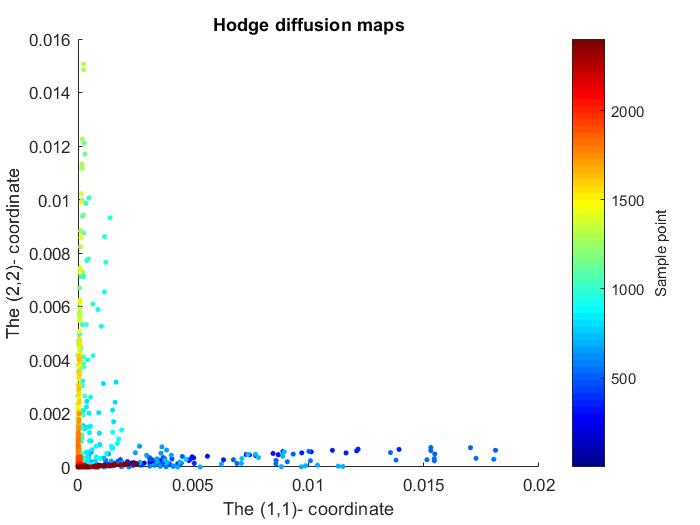}
\includegraphics[width=0.49\textwidth]{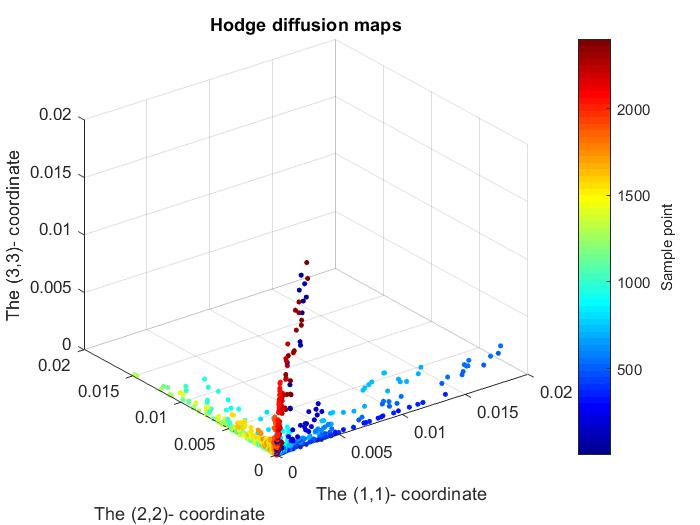}
\includegraphics[width=0.49\textwidth]{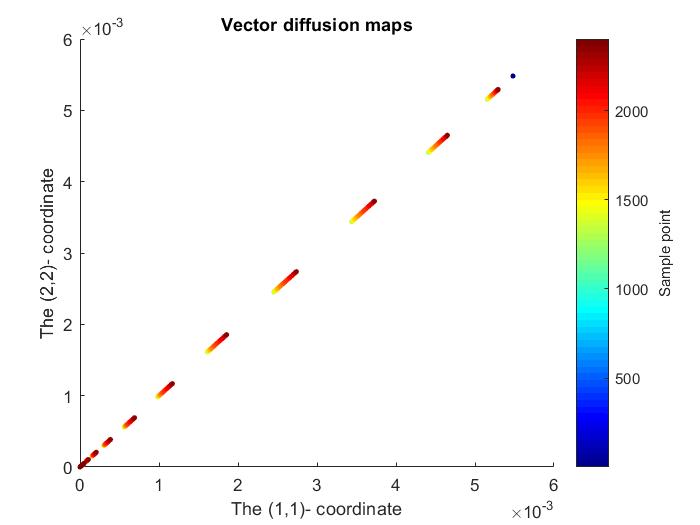}
\includegraphics[width=0.49\textwidth]{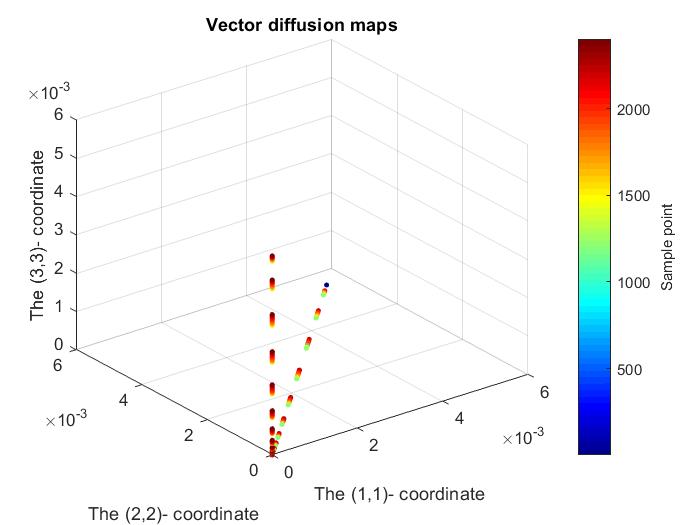}
\caption{Plot of the diagonal coordinates of the Hodge diffusion maps and vector diffusion maps. The first row shows the Hodge diffusion maps of the first order ($k=1$). The second row shows the second order ($k=2$) Hodge diffusion maps, and the third row shows the vector diffusion map. In the left column, we plot the first two diagonal coordinates, $(1,1)$ and $(2,2)$, and in the right column, we plot the first three diagonal coordinates, $(1,1)$, $(2,2)$, and $(3,3)$. The colorbar indicates the order of the points, matching the colorbar in Figure \ref{figuratoro}. The dataset consists of points sampled over the sphere.}
\label{embedingdiagesfera}
\end{figure}

\begin{figure}[htp]
\centering
\includegraphics[width=0.49\textwidth]{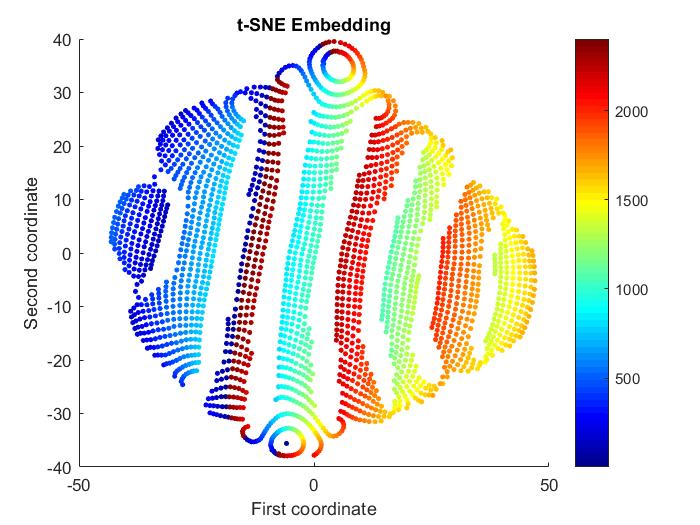}
\includegraphics[width=0.49\textwidth]{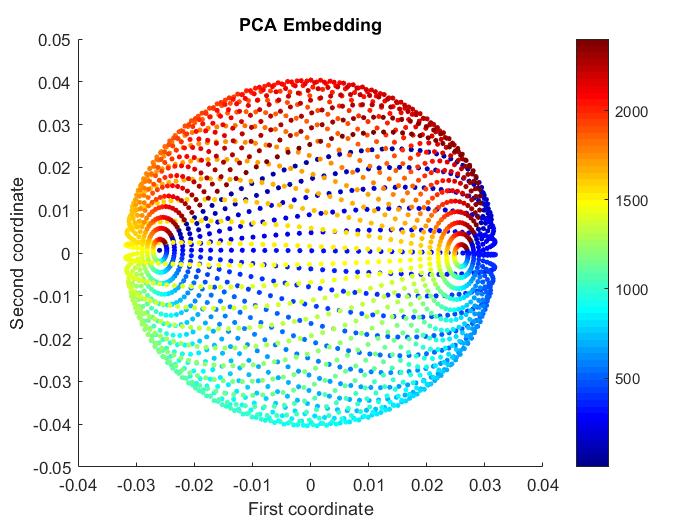}
\includegraphics[width=0.49\textwidth]{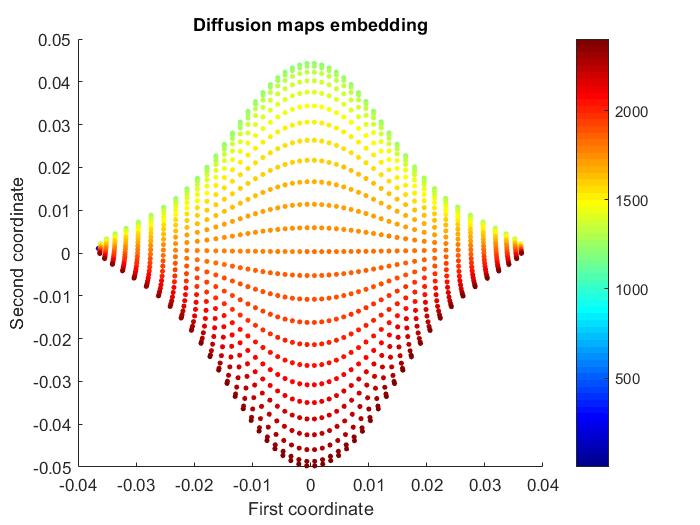}
\caption{Plot of the t-SNE, PCA, and Diffusion Maps algorithms applied to the dataset sampled on the sphere $S^2$.}
\label{otrosesfera}
\end{figure}

\section{Conclusions and future directions}
\label{conclusiones}
In this paper, we introduce Hodge diffusion maps, a generalization of both vector diffusion maps and classical diffusion maps. Assuming the dataset lies on a manifold, our algorithm leverages the $k$-th Hodge Laplacian—closely connected to the $k$-th cohomology group via Hodge theory—to extract topological features. While classical diffusion maps correspond to the computation of the zero-order Hodge Laplacian, and vector diffusion maps to the connection Laplacian (equivalent to the first-order Hodge Laplacian), our approach extends naturally to compute the Hodge Laplacian of any order $k\geq1$. This enables the extraction of richer topological information from the dataset. 

We validate our approach through two numerical experiments on datasets sampled from a torus and a sphere. In the first experiment (torus), both the proposed Hodge diffusion maps and classical diffusion maps successfully embed the vertical sections of the dataset as straight lines in both 2D and 3D Euclidean space, whereas the other methods fail to capture this structure. In the second experiment (sphere), only Hodge diffusion maps successfully separate each vertical section into distinct regions in both 2D and 3D embeddings, while classical diffusion maps do not achieve this distinction. Additionally, the embeddings produced by t-SNE and PCA map points with the similar topological structure -defined by the vertical sections- onto nonlinear curves. This complicates the use of linear classifiers to separate the data based on these vertical sections, further highlighting the superior performance of the proposed algorithm over these methods.

These experiments demonstrate that Hodge diffusion maps, as a dimensionality reduction technique, more effectively capture the topological structure of a dataset by mapping points with similar topological features to nearby regions in Euclidean space. This facilitates the use of linear classifiers to categorize the embedded points, underscoring Hodge diffusion maps as a valuable complementary tool for extracting additional topological information from the dataset.  

Based on the connection between diffusion map theory and vector diffusion maps with cryo-electron microscopy (Cryo-EM), as in Refs. \cite{singer2011three, singer2012vector, coifman2008graph}, we plan to explore the use of Hodge diffusion maps as a pre-processing tool for heterogeneous particles in future work. Another potential direction is to incorporate the Hodge-Laplacian matrix, as outlined in \cref{generalizationdifmaps}, as a penalty term for regularizing neural networks. This approach could enhance the network's ability to extract topological features from the dataset, thereby improving its robustness during training.

\bigskip

\noindent\textbf{Acknowledgements.} The first author was supported by Centro de Modelamiento Matemático (CMM) BASAL fund FB210005 for center of excellence from ANID-Chile. The second author was supported by Fondecyt ANID postdoctoral grant 3220631.

\bibliographystyle{alpha}

%\bibliographystyle{babalpha}
%bibliographystyle{apalike}
\bibliography{ref}

\newcommand{\etalchar}[1]{$^{#1}$}
\begin{thebibliography}{RGWC{\etalchar{+}}24}

\bibitem[CL06]{coifman2006diffusion}
Ronald~R Coifman and St{\'e}phane Lafon.
\newblock Diffusion maps.
\newblock {\em Applied and computational harmonic analysis}, 21(1):5--30, 2006.

\bibitem[CSSS08]{coifman2008graph}
Ronald~R Coifman, Yoel Shkolnisky, Fred~J Sigworth, and Amit Singer.
\newblock Graph laplacian tomography from unknown random projections.
\newblock {\em IEEE Transactions on Image Processing}, 17(10):1891--1899, 2008.

\bibitem[DC92]{do1992riemannian}
Manfredo~P Do~Carmo.
\newblock {\em Riemannian Geometry}.
\newblock Mathematics (Boston, Mass.). Birkh{\"a}user, 1992.

\bibitem[DC16]{do2016differential}
Manfredo~P Do~Carmo.
\newblock {\em Differential geometry of curves and surfaces.}
\newblock Courier Dover Publications, 2016.

\bibitem[GF25]{alvarorepo}
Alvaro~Almeida Gomez and Jorge~Duque Franco.
\newblock {\em Codes and numerical implementations of Hodge Diffusion Maps}, Apr. 3, 2025.
\newblock \url{https://github.com/alvaroalmeidagomez/HDM}.

\bibitem[GNZ21]{gomez2021diffusion}
Alvaro~Almeida Gomez, Ant{\^o}nio J~Silva Neto, and Jorge~P Zubelli.
\newblock Diffusion representation for asymmetric kernels.
\newblock {\em Applied Numerical Mathematics}, 166:208--226, 2021.

\bibitem[GNZ23]{10227282}
Alvaro~Almeida Gomez, Antônio J.~Silva Neto, and Jorge~P. Zubelli.
\newblock A diffusion-map-based algorithm for gradient computation on manifolds and applications.
\newblock {\em IEEE Access}, 11:90622--90640, 2023.

\bibitem[HHS{\etalchar{+}}23]{he2023learning}
Mingzhen He, Fan He, Lei Shi, Xiaolin Huang, and Johan~AK Suykens.
\newblock Learning with asymmetric kernels: Least squares and feature interpretation.
\newblock {\em IEEE Transactions on Pattern Analysis and Machine Intelligence}, 45(8):10044--10054, 2023.

\bibitem[HHYH23]{he2023diffusion}
Mingzhen He, Fan He, Ruikai Yang, and Xiaolin Huang.
\newblock Diffusion representation for asymmetric kernels via magnetic transform.
\newblock {\em Advances in Neural Information Processing Systems}, 36:53742--53761, 2023.

\bibitem[JLYY11]{jiang2011statistical}
Xiaoye Jiang, Lek-Heng Lim, Yuan Yao, and Yinyu Ye.
\newblock Statistical ranking and combinatorial hodge theory.
\newblock {\em Mathematical Programming}, 127(1):203--244, 2011.

\bibitem[Laf04]{lafon2004diffusion}
St{\'e}phane~S Lafon.
\newblock {\em Diffusion maps and geometric harmonics}.
\newblock Yale University, 2004.

\bibitem[Lee12]{lee2012introduction}
J.~Lee.
\newblock {\em Introduction to Smooth Manifolds}.
\newblock Graduate Texts in Mathematics. Springer New York, 2012.

\bibitem[Lim20]{lim2020hodge}
Lek-Heng Lim.
\newblock Hodge laplacians on graphs.
\newblock {\em Siam Review}, 62(3):685--715, 2020.

\bibitem[RGWC{\etalchar{+}}24]{ribando2024combinatorial}
Emily Ribando-Gros, Rui Wang, Jiahui Chen, Yiying Tong, and Guo-Wei Wei.
\newblock Combinatorial and hodge laplacians: Similarities and differences.
\newblock {\em SIAM Review}, 66(3):575--601, 2024.

\bibitem[SS11]{singer2011three}
Amit Singer and Yoel Shkolnisky.
\newblock Three-dimensional structure determination from common lines in cryo-em by eigenvectors and semidefinite programming.
\newblock {\em SIAM journal on imaging sciences}, 4(2):543--572, 2011.

\bibitem[SW11]{singer2011orientability}
Amit Singer and Hau-tieng Wu.
\newblock Orientability and diffusion maps.
\newblock {\em Applied and computational harmonic analysis}, 31(1):44--58, 2011.

\bibitem[SW12]{singer2012vector}
Amit Singer and H-T Wu.
\newblock Vector diffusion maps and the connection laplacian.
\newblock {\em Communications on pure and applied mathematics}, 65(8):1067--1144, 2012.

\bibitem[TSL23]{TOADER2023168020}
Bogdan Toader, Fred~J. Sigworth, and Roy~R. Lederman.
\newblock Methods for cryo-em single particle reconstruction of macromolecules having continuous heterogeneity.
\newblock {\em Journal of Molecular Biology}, 435(9):168020, 2023.
\newblock New Frontier of Cryo-Electron Microscopy Technology.

\bibitem[War83]{warner1983foundations}
Frank~W Warner.
\newblock {\em Foundations of differentiable manifolds and Lie groups}, volume~94.
\newblock Springer Science \& Business Media, 1983.

\bibitem[WWLX22]{wei2022hodge}
Ronald Koh~Joon Wei, Junjie Wee, Valerie~Evangelin Laurent, and Kelin Xia.
\newblock Hodge theory-based biomolecular data analysis.
\newblock {\em Scientific Reports}, 12(1):9699, 2022.

\end{thebibliography}

\appendix
\section{Alternanting forms and alternating arrays}
\label{appeA}
\label{aprendiceformas}
A vector subspace $V\subseteq\mathbb{R}^{n}$ is (non-canonically) isomorphic to its dual $V^*$. Fixing an inner product on $V$ induces a natural isomorphism $V\simeq V^*$, which in turn establishes corresponding isomorphisms between various spaces constructed from $V$ and $V^*.$ For instance, this yield isomorphism between
$$\underbrace{V \otimes \cdots \otimes V}_{k \text{-times}} \text{ and } \underbrace{V^* \otimes \cdots \otimes V^*}_{k \text{-times}},$$ 
as well as between the exterior powers
$$\bigwedge^k(V) \text{ and } \bigwedge^k(V^*).$$  

Since we are dealing with discrete data, a more concrete representation of the vector subspace $V$ and its associated constructions is needed for numerical implementation, rather than relying solely on its abstract definition. To address this, we introduce the notion of a $k$-dimensional real array of size $(n_1,n_2,\cdots,n_k).$ This approach provides a framework for defining alternating forms and, subsequently, differential forms in a manner that is better suited for numerical computations.

For every natural number $n$, we denote $I_n=\llave{1,2,\dots,n}$  and $S_n$ be the set of all the permutations of $I_n$. A $k$-dimensional real array of size $(n_1,n_2,\cdots,n_k)$ is defined as a function 
$$f: I_{n_1}\times \cdots \times I_{n_k} \to \mathbb{R}.$$ 
In particular, a $1$-dimensional array of size $(n)$ corresponds to a vector of $\R^n,$ while $2$-dimensional array of size $(n_1,n_2)$ corresponds to an $n_1\times n_2$ matrix. In this sense, $k$-dimensional array naturally generalize the notions of vectors and matrices. 

We denote the set of $k$-dimensional arrays of size  $(n_1,n_2,\dots,n_k)$ by $M(n_1,n_2,\dots,n_k)$.
Given two arrays, one of dimension \( k \) (denoted \( f \in M(n_1, n_2, \dots, n_k) \)) and the other of dimension \( l \) (denoted \( g \in M(m_1, m_2, \dots, m_l) \)), we define their tensor product \( f \otimes g \) as a \( k+l \)-dimensional array in \( M(n_1, n_2, \dots, n_k, m_1, m_2, \dots, m_l) \), given by
\begin{equation*}
    f \otimes g (i_1,\cdots,i_k,j_1,\cdots,j_l)=f(i_1,\cdots,i_k)g(j_1,\cdots,j_l)
\end{equation*}
with $i_s \in I_{n_s}$ and $j_{\hat{s}} \in I_{m_{\hat{s}}}$ with $1 \le s \le k$ and $1 \le \hat{s} \le l$.
We endow the space $M(n_1,n_2,\dots,n_k)$ with the Frobenius inner product, defined as
 \begin{equation}\label{productointernoarreglos}
     <A,B>_{F}= \sum_{i_1=1}^{n_1} \sum_{i_2=1}^{n_2} \cdots \sum_{i_k=1}^{n_k} A(i_1,i_2, \cdots, i_k) B(i_1,i_2, \cdots, i_k)
 \end{equation}
 for all $A,B \in M(n_1,n_2,\dots,n_k)$. Now observe that $M(n)$ corresponds to $\mathbb{R}^n,$ so any linear subspace $V$ of $\mathbb{R}^n$ naturally defines a subspace of $M(n).$ Given an $d$-dimensional linear subspace $V\subseteq \mathbb{R}^n$, the tensor product space  
 $$\underbrace{V \otimes \cdots \otimes V }_{k-times}$$ is identified with the linear subspace of $M(\underbrace{n,\dots,n}_{k-times})$ spanned by the elements
\begin{equation*}
    \{ v_1 \otimes v_2 \otimes \cdots \otimes v_k | v_i \in V. \}
\end{equation*}
Here, $V$ is identified with its corresponding subspace in $M(n).$ 

\begin{definition}
\label{def:form2811}        
    A $k-$alternating form $\omega:\underbrace{V \times V \times \cdots \times V}_{k-times} \to \mathbb{R}$ defined over a vector space $V$ is a multilinear map such that is alternating, that is, if for all vectors $v_1,v_2,\cdots,v_k$ and and any permutation $\sigma \in S_k$
    \begin{equation*}
        \omega^{\sigma}(v_1,\dots,v_k):=\omega(v_{\sigma(1)},\dots,v_{\sigma(k)})=(\text{sign} \, \sigma) \omega(v_1,\dots,v_k).
        \end{equation*}
    We denote the set of $k$-alternating forms as $\bigwedge^k(V^*)$
\end{definition}
Since we have the Frobenius product, it induces an isomorphism between 
$$\underbrace{V \otimes V \otimes \cdots \otimes V}_{k-times} \text{ and } \underbrace{(V \otimes V \otimes \cdots \otimes V)^*}_{k-times}.$$ 
As a result, each alternating form $\omega$ corresponds uniquely to a $k$-dimensional array $W$, allowing us to switch between $\omega$ and $W$  using this isomorphism. We introduce the following definition:

\begin{definition}\label{def:alternatingmulti}
  A $k$-dimensional array $ f \in \underbrace{V \otimes V \otimes \cdots \otimes V}_{k-times}\subseteq M(\underbrace{n,\dots,n}_{k-times})$  is called \( k \)-alternating array in \( V \) if it satisfies the following condition:
\begin{itemize}
    \item \textbf{C1.} For all indices $ i_1,i_2,\cdots,i_k \in I_n$ and all permutation $\sigma \in S_k$, we have:
$f(i_{\sigma(1)},i_{\sigma(2)},\cdots,i_{\sigma(k)})=(\text{sign} \, \sigma)f(i_1,i_2,\cdots,i_k)$.
\end{itemize}
We denote the set of $k$-alternating arrays as $\Theta^k(V)$.
\end{definition}

 An important property is that any $k$-dimensional array $f \in \underbrace{V \otimes \cdots \otimes V}_{k \text{-times}}$, even if it is not alternating, satisfies the following property, whose proof is straightforward

 \begin{pro} \label{propiedad:alternante}
For every permutation $\sigma \in S_{k}$, we let $f^{\sigma}(i_1, \cdots, i_{k}):=f(i_{\sigma(1)}, \cdots, i_{\sigma(k)})$, then for any vectors  $v_1,\cdots, v_k \in V$ we have:
\begin{equation*}
     <f^{\sigma}, v_1 \otimes v_2 \cdots \otimes v_{k}>_{F}=<f,v_{\sigma(1)} \otimes v_{\sigma(2) } \cdots \otimes v_{\sigma(k)} >_{F}
\end{equation*} 
\end{pro}

\begin{prop}
\label{proposicion:universal}
    Let $V$ a $d$-dimensional linear subspace of $\mathbb{R}^n$ and $\omega \in \bigwedge^k(V^{*}
    ) $ an alternating $k$-form. Then, there exists an unique $k$-alternating array $W \in \Theta^k(V)$ such that:
    \begin{equation}\label{eq:tensorproductoexistencia}
        \omega(v_1,v_2,\cdots,v_k)= < W, v_1 \otimes v_2 \cdots \otimes v_k>_{F}
    \end{equation} 
    where $< \cdot , \cdot>$ is the Frobenius inner product for arrays defined as in \cref{productointernoarreglos}.
\end{prop}

\begin{proof}
This follows from the fact that the isomorphism between $V$ and $V^*$ induced by an inner product $(\cdot,\cdot)$ is given explicitly by the mapping 
$$v\in V\to v^*\in V^*, \text{ where } v^*(t):=(v,t).$$ More precisely, using the formalism introduced so far: By the universal property of the tensor product (\cite[Proposition 12.7]{lee2012introduction}), there exists a unique linear map 
$$L: V \otimes \cdots \otimes V \to \mathbb{R}$$ 
such that the following commutative diagram holds:
\begin{figure}[H]
\centering
\includegraphics[width=.41\textwidth]{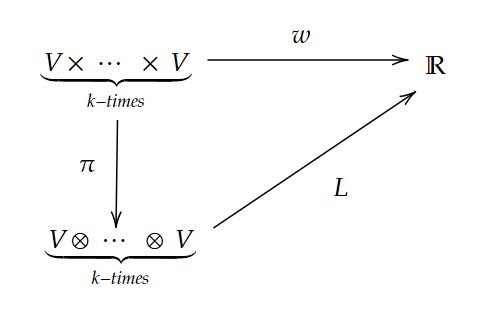}
\caption{Commutative diagram for the functions $w$, $L$ and $\pi$.}
\label{fig:teoremauniversal}
\end{figure}
Here, the map 
$$\pi:\underbrace{V \times \cdots \times V}_{k-times} \to \underbrace{V \otimes \cdots \otimes V}_{k-times}$$ is defined by 
$$\pi(v_1,\cdots,v_k)=v_1\otimes \cdots \otimes v_k.$$ By the Riesz representation theorem, there exists a unique $k$-dimensional array $W \in \underbrace{V \otimes \cdots \otimes V}_{k-times}$ such that 
$$L(v_1 \otimes v_2 \cdots \otimes v_k)=<W,v_1 \otimes v_2 \cdots \otimes v_k>_{F}$$  
for all $v_1,v_2, \dots, v_k \in V.$ This together with the commutative diagram proves \cref{eq:tensorproductoexistencia}. To show that $W$ is alternating, let $\sigma \in S_k$ be a permutation and consider the $k$-dimensional array $W^\sigma$ as in \cref{propiedad:alternante}. Then by the same property, we obtain
\begin{equation*}
\begin{array}{rcl}
     <W^{\sigma}, v_1 \otimes v_2 \cdots \otimes v_{k}>_{F}&=&<W,v_{\sigma(1)} \otimes v_{\sigma(2) } \cdots \otimes v_{\sigma(k)} >_{F}\\
     \, & = & w(v_{\sigma(1)}, v_{\sigma(2) }, \cdots ,v_{\sigma(k)}) \\
    \, & =  &(\text{sgn} \, \sigma) w(v_1,\cdots,v_k)
     \end{array}.
\end{equation*}
By the uniqueness of $W,$ it follows that $W^{\sigma}=(\text{sgn} \, \sigma)W,$ which completes the proof.
\end{proof}

\begin{examp}
    As an example illustrating \cref{proposicion:universal}, consider the determinant  as alternating form on $\R^n.$ For any vector $v_1,v_2, \cdots, v_n$, we have
    \begin{equation*}
    \begin{array}{rcl}
        \det(v_1,v_2, \cdots, v_n) & = & \sum_{\sigma \in S_n} (\text{sgn} \, \sigma) v_{1}(\sigma(1))v_{2}(\sigma(2)) \cdots v_{n}(\sigma(n)) \\
        \, & = & < W, v_1 \otimes v_2 \cdots \otimes v_n>_{F} \\
         \end{array}
    \end{equation*}
    where the array $W(i_1, \cdots,i_n)$ is defined as $\text{sgn} \, \sigma $ if there exist a permutation $\sigma $ with $\sigma(s)=i_s$ for $1\le s \le n$ and $0$ otherwise. In the two dimensional case $\R^2$ the $2$-alternating array $W$ is given by
 $$\begin{pmatrix}
  0 & 1\\ 
  -1 & 0
\end{pmatrix}.$$
    \end{examp}

\begin{rem}\label{remarkisomorfismo}
\cref{proposicion:universal} establishes an isomorphism between $k$-alternating forms $\bigwedge^k(V^*)$ and $k$-alternating arrays $\Theta^k(V),$ which $I: \bigwedge^k(V^*) \to \Theta^k(V).$
\end{rem}

\begin{rem}\label{observacionprojectionalternada}
    Given an alternating form $\omega \in \bigwedge^k(\R^n)^*$ and a linear subspace $V$ of $\R^n$, the restriction $\omega |_{V \times \cdots\times V}$ induces an alternating form in $\bigwedge^k(V)$. In this case, it is straightforward to show that $I(\omega |_{V \times \cdots\times V})= \mathcal{P}_{V \otimes \cdots \otimes V}(I(\omega))$, where $\mathcal{P}_{V \otimes \cdots \otimes V}$ is the orthogonal projection onto the subspace $V \otimes \cdots \otimes V$. Therefore for any $k$-alternating array $W \in \Theta^k(\R^n)$ the projection onto the tensor space ${V \otimes \cdots \otimes V}$ induces an $k$-alternating array on the linear subspace $V$, which is given by $\mathcal{P}_{V \otimes \cdots \otimes V}(W) \in \Theta^k(V)$.
 \end{rem}

Via the isomorphism $I,$ we can compute the wedge product of $k$-alternating forms using $k$-alternanting arrays. Consequently, all possible (discrete) computations of differential $k$-forms will inherently rely on $k$-alternating arrays, if this isomorphism is not explicitly mentioned. For instance, let $\omega_1$ be a $k_1$-alternanting form and $\omega_2$ a $k_2$-alternating form. Recall that their wedge product $w_1 \wedge w_2 $ is a $(k_1+k_2)$-alternanting form given by
\begin{equation*}
    \omega_1 \wedge \omega_2=\frac{1}{k_1! k_2!} \sum_{\sigma \in S_{k_1+k_2}} (\text{sgn} \, \sigma) (\omega_1 \otimes \omega_2)^{\sigma}, 
\end{equation*}
where $\omega_1 \otimes \omega_2(u,v)=\omega_1(u) \omega_2(v)$, and $\omega^{\sigma}$ is defined as     in \cref{def:form2811}. In this framework, we obtain the compatibility relation  
\begin{equation*}
    I(\omega_1 \wedge \omega_2)=I(\omega_1) \wedge I(\omega_2),
\end{equation*}
where the wedge product on the right-hand side is defined in the same manner as in differential forms. Naturally, this identity can be verified directly using the isomorphism $I$ and the previously established properties. We illustrate this in the following proposition: 

\begin{prop}
\label{isobundles}
Let $\omega_1 \in \Lambda^{k_1}(V^*) $ and $\omega_2 \in \Lambda^{k_2}(V^*)$ then:

\begin{equation*}
    I(\omega_1 \wedge \omega_2)=I(w_1) \wedge I(w_2)
\end{equation*}
\end{prop}

\begin{proof}
  Observe that for all $v_1,\cdots, v_{k_1}, v_{k_1+1}, \cdots, v_{k_1+k_2} \in V$, we have, by \cref{propiedad:alternante}

    \begin{equation*}
    \begin{split}
        &<I(\omega_1)\wedge I(\omega_2), v_1\otimes \cdots \otimes v_{k_1} \otimes v_{k_1+1} \otimes \cdots \otimes v_{k_1+k_2}>_{F}\\
       &= \frac{1}{k_1! k_2!} \sum_{\sigma \in S_{k_1+k_2}} (\text{sgn} \, \sigma) <I(\omega_1) \otimes I(\omega_2), v_{\sigma (1)}\otimes \cdots \otimes v_{\sigma (k_1)} \otimes v_{\sigma (k_1+1)} \otimes \cdots \otimes v_{\sigma (k_1+k_2)}>_{F}\\
       &=\frac{1}{k_1! k_2!} \sum_{\sigma \in S_{k_1+k_2}} (\text{sgn} \, \sigma) <I(\omega_1) , v_{\sigma (1)}\otimes \cdots \otimes v_{\sigma (k_1)} >   < I(\omega_2),  v_{\sigma (k_1+1)} \otimes \cdots \otimes v_{\sigma (k_1+k_2)}>_{F}\\
       &=\frac{1}{k_1! k_2!} \sum_{\sigma \in S_{k_1+k_2}} (\text{sgn} \, \sigma) (\omega_1 \otimes \omega_2)^{\sigma} ( v_{1}, \cdots , v_{k_1} , v_{k_1+1} , \cdots , v_{k_1+k_2})\\
       &= (\omega_1 \wedge \omega_2) (v_1,\cdots, v_{k_1}, v_{k_1+1}, \cdots, v_{k_1+k_2}).
        \end{split}
    \end{equation*}
    Since $I(\omega_1) \wedge I(\omega_2)$ belongs to $\Theta^{k_1+k_2}(V)$, the uniqueness of \cref{eq:tensorproductoexistencia} guarantees that 
    $$I(\omega_1) \wedge I(\omega_2)= I(w_1 \wedge w_2).$$ 
\end{proof}
We conclude this section by recalling the following result, which will be used in various calculations: If $v_1, v_2, \dots, v_d$ form an orthonormal basis for  $V$, then the set of wedge products
\begin{equation}
\label{wpbase}
    \llave{ \frac{1}{\sqrt{k!}}v_{i_1}^{*} \wedge v_{i_2}^{*}  \wedge \cdots \wedge v_{i_k}^{*} \mid i_1<i_2< \cdots < i_k }
\end{equation}
constitutes an orthonormal basis for $\Lambda^{k}(V^*)$, where the inner product on $\Lambda^{k}(V^*)$ is given by
\begin{equation*}
    <\omega_1,\omega_2>_{\Lambda^{k}(V^*)}=<I(\omega_1), I(\omega_2)>_{F}
\end{equation*}

\section{Proof of \cref{teoremaprincipal}}
\label{apendicepruebaprincipal}
%\label{ape1}
%\label{appecB}
%\label{ape1}
In this section, we present the technical details supporting \cref{teoremaprincipal}. The proof builds upon the framework developed in \cite{10227282}, with several components adapted to fit our setting. For additional background and a more comprehensive treatment of the underlying concepts, we refer the reader to \cite{10227282}, as well as to \cite{do1992riemannian, do2016differential} for a thorough introduction to differential geometry.

Recall that $\mathcal{M}$ is a closed (i.e., compact without boundary) Riemannian manifold and let $x \in \mathcal{M}.$ For a small positive real number $\varepsilon$, consider the map $\psi=\exp_{x} \circ \,T : B(0,\varepsilon)\subset \R^d \to \mathcal{M}$ which defines a normal coordinate system around the point $x$. Here, $\exp_{x}$ denotes the exponential map at $x$, and $T:\R^d \to T_{x} \mathcal{M}$ is a rotation from $\R^d$ onto the tangent space $T_{x} \mathcal{M}$, which is considered as subset of $\R^n$. Note that $\psi(0)=x$. We now recall some estimates in normal coordinates system that are useful for approximating differential operators. The Taylor expansion of $\psi$ around the point $0$ is given by
\begin{equation}
    \psi(v)= x+ T (v)+\frac{1}{2} D^2\psi_0(v,v)+ O(\|  v\|^3),
    \label{taylorexpo}
\end{equation}
where $D^2\psi_0$ denotes the second order differential (also known as the Hessian) of $\psi$ at $0$. Let $v \in B(0,\varepsilon)\subset \R^d$, and consider the geodesic $\gamma_{T(v)}$,  with initial tangent vector $T(v) \in T_{x} \mathcal{M}$. Then, the expansion in \cref{taylorexpo} can be rephrased in terms of the geodesic as  
\begin{equation}
   \gamma_{T(v)}(t)=x+T(v)\,t+\frac{1}{2} D^2\psi_0\,(v,v) t^2+O(\|v\|^3)t^3,
   \label{taylorgeodesica}
\end{equation}
for $t \in \R$. Since the covariant derivative of a geodesic vanishes, we have that $\gamma''_{{T(v)}}$ is orthogonal to $T_{x} \mathcal{M}$. Therefore, from \cref{taylorexpo,taylorgeodesica}, we obtain the following estimates
\begin{equation}
\|\psi(v)-x\|^2= \|T(v)\|^2+O(\|v\|^4),
\label{estimativaorden2}
\end{equation}
and
\begin{equation}
   \mathcal{P}_{T_{x} \mathcal{M}}(\psi(v)-x)= T(v)+O(\|v\|^3),
  \label{estimativaorden3}
\end{equation}
where $\mathcal{P}_{T_{x} \mathcal{M}}$ denotes the orthogonal projection onto the tangent space  $T_{x} \mathcal{M}$. Here we have taken advantage of the fact that the manifold $\mathcal{M}$ is embedded in $\R^n$. Moreover, letting $e_1, \cdots, e_d$  be the standard basis in $\R^d$, then by differentiating \cref{estimativaorden3} with respect to the variable $v_i$ we obtain:
\begin{equation}
   \label{taylorexpansiondifferential}
     \mathcal{P}_{T_{x} \mathcal{M}}\paren{\frac{\partial \psi}{\partial v_i}(v)}= T(e_i)+O(\|v\|^2),
\end{equation}
Using Estimates~(\ref{estimativaorden2}) and~(\ref{estimativaorden3}), we conclude that there exist positive constants $M_1$ and $M_2$ such that, for $\|v\|$ small
$$\|v\|-M_2 \|v\|^3 \le \|\psi(v)-x\|\le M_1 \|v\|.$$
In particular, if $\|v\|^2 \le \frac{1}{2M_2}$, then
$$\frac{1}{2}\|v\|  \le \|\psi(v)-x\|\le M_1 \|v\|.$$
This implies that, for small $t$, we have the following inclusion:
\begin{equation}
    \label{condilocal}
    B(0,t/{M_1}) \subseteq \psi^{-1} (U(x,t^\delta)) \subseteq B(0,2 t).
\end{equation}
 where $U(x,t^\delta)$ denotes the ball in $\mathcal{M}$ centered at $x$ with radios $t^\delta$, that is 
 $$U(x,t^\delta):=\{ y \in \mathcal{M} | \|y - x \| \le t^\delta \}.$$

\subsection{Expansion of the Operator in \cref{operadorP}}
\label{ape2}
In this section, we continue with the technical development of the proof of \cref{teoremaprincipal}. The central idea is to apply the Taylor expansion of the differential form $w$ around the point $p$. To this end, we present a sequence of lemmas that progressively build toward the main result, which will be established at the end of the section.

\begin{lem}
\label{lemaprinci}
Assume that $\frac{1}{2} < \delta<1$, and let $K:\mathcal{M} \times \mathcal{M} \to \R^m$  be a vector value kernel. Define
$$P_{t,\delta}(x)=\int_{U(x,t^\delta)} K(x,y) \, e^{-\frac{\| y-x \|^2}{2 t^2}} dVol(y),$$
 where the integration is performed componentwise for the vector-valued function. Suppose that for small $t$, the function $\psi:B(0,2t^\delta) \to \mathcal{M} $ defines a normal coordinate system in a neighborhood of $x$. Let $S: \R^d \to \R^m$ be a vector value function such that
$$K(x,\psi(v))-S(v)=O(\|v\|^r),$$
and
$$K(x,y)=O(\|x-y\|^s).$$
Then, the following estimate holds:
$$ P_{t,\delta}(x)=O(   (e^{C_2 t^{4\delta-2}} -1) t^{s+d}+ t^{r+d})+\int_{\psi^{-1} (U(x,t^\delta))} S(v) \, e^{\frac{-\| T(v)\|^2}{2 t^2}} dv.$$
where $T$ is a rotation from $\R^d$ onto the tangent space $T_{x} \mathcal{M}$.
\end{lem}
\begin{proof}
Using \cref{condilocal}, we assume that for small $t$, the set $U(x,t^\delta)$ lies within the image of a normal chart $\psi:B(0,2t^{\delta}) \to \mathcal{M} $ centered in $x$. Therefore, we can write:
$$\begin{array}{rcl}  \int_{U(x,t^\delta)} K(x,y) \,  e^{\frac{-\| y-x \|^2}{2 t^2}} dVol(y) & = & \int_{\psi^{-1} (U(x,t^\delta))} K(x,\psi(v)) e^{\frac{-\| \psi(v)-x \|^2}{2 t^2}} dv \\ \, & = & \int_{\psi^{-1} (U(x,t^\delta))}  K(x,\psi(v)) (e^{\frac{-\| \psi(v)-x \|^2}{2 t^2}}-e^{\frac{-\| T(v) \|^2}{2 t^2}}) dv \\ \, & + & \int_{\psi^{-1} (U(x,t^\delta))} (K(x,\psi(v))-S(v)) e^{\frac{-\| T(v)\|^2}{2 t^2}} dv
\\  \, & + & \int_{\psi^{-1} (U(x,t^{\delta}))} S(v) \, e^{\frac{-\| T(v)\|^2}{2 t^2}} dv.
\end{array}  $$
We now estimate the first term, which we denote by
$$A:=\int_{\psi^{-1} (U(x,t^\delta))} K(x,\psi(v)) (e^{\frac{-\| \psi(v)-x \|^2}{2 t^2}}-e^{\frac{-\| T(v) \|^2}{2 t^2}}) dv. $$
Using \cref{estimativaorden2}, and the inequality $ |e^x-1|\le e^{|x|}-1$, we obtain
$$\begin{array}{rcl}  |e^{\frac{-\| \psi(v)-x \|^2}{2 t^2}}-e^{\frac{-\| T(v) \|^2}{2 t^2}}| & = & e^{\frac{-\| T(v) \|^2}{2 t^2}} | e^{\frac{O(\|v\|^4)}{2 t^2}} -1| \\ \, & \le & e^{\frac{-\| T(v) \|^2}{2 t^2}} ( e^{\frac{C_1\|v\|^4}{2 t^2}} -1).
\end{array}  $$
Therefore, by \cref{condilocal} we obtain 
$$\begin{array}{rcl} \| A\| & \le & C_3 \, \, t^s ( e^{C_2 t^{4\delta-2}} -1) t^d \int_{\R^d} \| v \|^s e^{-\| v \|^2 /2}  dv \\ \, & = & O( (e^{C_2 t^{4\delta-2}} -1) t^{s+d} ).
\end{array}  $$
On the other hand, by assumption we have 
$$\int_{\psi^{-1} (U(x,t^\delta))} ( K(x,\psi(v))-S(v)) e^{\frac{-\| T(v)\|^2}{2 t^2}} dv=O( t^{r+d} \,).$$
\end{proof}

\begin{lem} 
\label{lemma2}
Under the same assumptions as in \cref{lemaprinci}, consider the integral 
$$E:=\int_{\psi^{-1} (U(x,t^\delta))} Q(v) e^{\frac{-\| T(v)\|^2}{2 t^2}} g(v)dv,$$
where $g$ is a smooth function at $0$ and  $Q$ is a homogeneous polynomial of degree $l$. Then, we have the following estimate:
$$E=\int_{\R^d} Q(v) e^{\frac{-\| T(v)\|^2}{2 t^2}} \paren{g(0)+\sum_{i=1}^{d} \frac{\partial g}{\partial v_i} (0) v_i} dv\,+ O(t^{d+l} e^{-M_2 t^{2(\delta-1)}}+t^{d+2+l}).$$
\end{lem}

\begin{proof}
Using the Taylor expansion of $g$ around $0$ we have
$$E=\int_{\psi^{-1} (U(x,t^\delta))} Q(v) e^{\frac{-\| T(v)\|^2}{2 t^2}} \paren{g(0)+\sum_{i=1}^{d} \frac{\partial g}{\partial v_i} (0)\,   v_i+ O(\|v\|^2) }dv .$$
Next, define 
$$ B:= \norma{\int_{\R^d \setminus \psi^{-1} (U(x,t^\delta))} Q(v) e^{\frac{-\| T(v)\|^2}{2 t^2}} \paren{g(0)+\sum_{i=1}^{d} \frac{\partial g}{\partial v_i} (0)\,  v_i} dv}.$$
Using \cref{condilocal} and the rapid decay of the exponential function, we obtain the estimate
$$ B \le  C_4  t^{d+l} e^{-M_2 t^{2(\delta-1)}} \int_{\R^d \setminus B(0 ,  t^{\delta-1}/M_1)}  P(\|v\|) e^{\frac{-\| T(v)\|^2}{4}} dv.$$
for some polynomial $P$. Hence
$$B = O(t^{d+l} e^{-M_2 t^{2(\delta-1)}} ),$$
for an appropriate constant $M_2>0$. Finally, we observe that the contribution of the remainder term in the Taylor expansion satisfies
$$\int_{\psi^{-1} (U(x,t^\delta))} Q(v) e^{\frac{-\| T(v)\|^2}{2 t^2}} O(\|v\|^2) dv=O(t^{d+2+l}  ).$$
\end{proof}

\begin{lem} 
\label{lema3desintegral}
    Under the same assumptions as in \cref{lemaprinci}, we obtain the following result
\begin{equation*}
       \int_{\mathcal{M}} K(x,y) \, e^{-\frac{\| y-x \|^2}{2 t^2}} dVol(y)= P_{t,\delta}(x) + O (t^{s+2(1-\delta)(d+2)}),
    \end{equation*}
\end{lem}
\begin{proof}
By assumption, the expression
\begin{equation*}
   \norma{ \int_{ \mathcal{M} \setminus U(x,t^\delta)} K(x,y) \, e^{\frac{-\| y-x \|^2}{2 t^2}} dVol(y) }= \norma{ \int_{\mathcal{M}} K(x,y) \, e^{\frac{-\| y-x \|^2}{2 t^2}} dVol(y)-P_{t,\delta}(x)}
\end{equation*}
is bounded from above by
\begin{equation}
\label{ecuacion1L3}
    F_1 \int_{ \mathcal{M} \setminus U(x,t^\delta)} \|x-y \|^{s} e^{\frac{-\| y-x \|^2}{2 t^2}} dVol(y)
\end{equation}
for some constant $F_1>0$. Since the exponential decay dominates the polynomial growth at infinity, there exists a constant $F_2>0$ such that for all $z\in \R^n$
    \begin{equation*}
        \|z \|^{s+2(d+2)} e^{-\frac{\|z\|^2}{2}} \le F_2
    \end{equation*}
Therefore, the expression in \cref{ecuacion1L3} is bounded from above by
\begin{equation*}
\begin{aligned}
      &F_1 F_2 \int_{ \mathcal{M} \setminus U(x,t^\delta)} \frac{t^{s+2(d+2)}}{\|x-y \| ^{2(d+2)}} dVol(y) \\
      &\le  F_1 F_2 \int_{ \mathcal{M} \setminus U(x,t^\delta)} t^{s+2(1-\delta)(d+2)} dVol(y) \\
      &\le F_1 F_2 t^{s+2(1-\delta)(d+2)} Vol( \mathcal{M}).  
\end{aligned}
\end{equation*}
\end{proof}

We recall some standard computations involving the moments of the Gaussian distribution, which will be useful in the proof of \cref{teoremaprincipal}.
For all index $i$, we have 
$$\int_{\R^d} v_i e^{\frac{-\| T(v)\|^2}{2 t^2}}dv=0,$$
and
$$\int_{\R^d} v_i^2 e^{\frac{-\| T(v)\|^2}{2 t^2}}dv=(2 \pi)^{\frac{d}{2}}\, t^{d+2}.$$
Moreover, if $i \neq j,$ 
$$\int_{\R^d} v_i \, v_j e^{\frac{-\| T(v)\|^2}{2 t^2}}dv=0.$$
These identities show that all odd moments vanish, and only the even-order moments contribute significantly. Consequently, we will focus on the even moments of the Gaussian distribution in what follows.

\begin{lem} 
\label{lemcalculosimple} 
Let $x\in \mathcal{M}$, and suppose $h: \mathcal{M} \to \R$ is a smooth function in $x$. Then 
\begin{equation*}
    \int_{\mathcal{M}}  e^{-\frac{\| y-x \|^2}{2 t^2}} h(y) dVol(y)= (2 \pi)^{\frac{d}{2}} t^d h(x)+O(t^{d+4 \delta -2})+O (t^{2(1-\delta)(d+2)}),
\end{equation*}
\end{lem}

\begin{proof}
Let $\psi$ be the map that defines normal coordinates at the point $x \in \mathcal{M}$. We apply Lemmas \ref{lemaprinci}, \ref{lemma2}, and \ref{lema3desintegral} to the functions $K(x,y)=h(y)$, $S(v)=h(\psi(v))$, $Q(v)=1$, and $g(v)=h(\psi(v))$, using the parameters $r=2$, $s=0$ and $l=0$. For any $ \frac{1}{2} < \delta < 1$, \cref{lema3desintegral} guarantees:
\begin{equation*}
    \int_{\mathcal{M}}  e^{-\frac{\| y-x \|^2}{2 t^2}} h(y) dVol(y)= P_{t,\delta}(x) + O (t^{2(1-\delta)(d+2)})
\end{equation*}
On the other hand, by applying Lemmas \ref{lemaprinci} and \ref{lemma2}, and using the rapid decay of the exponential function, we find that
\begin{equation*}
    P_{t,\delta}(x)=  h(x) \int_{\R^d} e^{\frac{-\| T(v)\|^2}{2 t^2}} dv +O(t^{d+4 \delta -2})
\end{equation*}
We remark that in \cref{lemma2}, the integral involving the first-order partial derivatives of $g(v)$ vanishes due that the odd moments of the Gaussian are zero. Additionally, note that in the proof of this lemma, we used the fact that $g(0)=h(x)$.
\end{proof}

\begin{lem} 
\label{lemanor}  Under the same assumptions as in Lemmas \ref{lemaprinci} and \ref{lemma2}, we obtain the following estimate for any \( \frac{1}{2} < \delta < 1 \):
\begin{equation*}
    %\label{estimativanormali}
    d_t(x)= (2 \pi)^{\frac{d}{2}} t^d q(x)+O(t^{d+4 \delta -2})+O (t^{2(1-\delta)(d+2)}),
\end{equation*}
where $d_t$ is defined in \cref{funcionnormalizacion}.

\end{lem}
\begin{proof} 
This follows directly from \cref{lemcalculosimple} applied to the function  \( h(y) = q(y) \).
\end{proof}

%\section{Proof of Theorem \ref{teoremaprincipal}}
%\label{apendicepruebaprincipal}

With these lemmas established, we are now ready to prove \cref{teoremaprincipal}. The proof is presented below.

\bigskip

\noindent\textbf{Proof of \cref{teoremaprincipal}} 
Let $e_1,\cdots,e_d$ be the standard basis of $\R^d$, and let $\psi$ be the map that defines the normal coordinates at the point $x \in \mathcal{M}$. Using the normal coordinate system, the $k$-differential form $w$ can locally be written as: 

\begin{equation*}
%\label{expansionwforma1}
    w(\psi(v))=\sum_{I} a_I(v) \frac{ \partial \psi}{\partial v_{i_1}}(v) \wedge  \cdots \wedge \frac{ \partial \psi}{\partial v_{i_k}}(v)
\end{equation*}
Moreover, since $\frac{ \partial \psi}{\partial v_{i_j}}(v)= \mathcal{P}_{{T_{x} \mathcal{M}}}\paren{\frac{ \partial \psi}{\partial v_{i_j}}(v)} + \mathcal{P}_{{T_{x} \mathcal{M}}^\perp}\paren{\frac{ \partial \psi}{\partial v_{i_j}}(v)}$, we can expand the previous expression as

\begin{equation} 
\label{expansionwedgeprueba1}
    w(\psi(v)) =\sum_{I} a_I(v) \mathcal{P}_{{T_{x} \mathcal{M}}}\paren{\frac{ \partial \psi}{\partial v_{i_1}}(v)} \wedge \cdots \wedge\mathcal{P}_{{T_{x} \mathcal{M}}}\paren{\frac{ \partial \psi}{\partial v_{i_k}}(v)} +L, 
\end{equation}
where $L$ is the remaining term which involves the wedge product of some term of the orthogonal complement $\mathcal{P}_{{T_{x} \mathcal{M}}^\perp}\paren{\frac{ \partial \psi}{\partial v_{i_j}}(v)}$. Next, we use \cref{taylorexpansiondifferential} to further expand \cref{expansionwedgeprueba1}: 

\begin{equation*} 
    w(\psi(v)) =\sum_{I} a_I(v) T(e_{i_1}) \wedge \cdots \wedge T(e_{i_k}) +L + O(\|v\|^2).
\end{equation*}
Thus, the difference between the two forms $w(\psi(v))-w(x)$, both viewed as multidimensional arrays in $\R^n$ can be expanded as follows:

\begin{equation}
\label{expansionwedgepruebafinal}
    w(\psi(v))-w(x)=\sum_{I} (a_I(v)-a_I(0)) T(e_{i_1}) \wedge \cdots \wedge T(e_{i_k}) +L + O(\|v\|^2).
\end{equation}
Since the term $L$ involves the wedge product of elements in the orthogonal complement ${T_x \mathcal{M}}^\perp$, the orthogonal projection onto $\bigwedge^k T_x \mathcal{M}$ vanishes:
\begin{equation}
\label{expansionauxiliar1}
    \mathcal{P}_{T_x \mathcal{M} \wedge \cdots \wedge T_x \mathcal{M}} ((\psi(v)-x)\wedge L)=0.
\end{equation}
By a similar argument, the following projection also vanishes: 

\begin{equation}
\label{expansionauxiliar2}
    \mathcal{P}_{\bigwedge^{k+1} T_x \mathcal{M}} ((\mathcal{P}_{{T_{x} \mathcal{M}}^\perp} (\psi(v)-x))\wedge ( w(\psi(v))-w(x)))=0.
\end{equation}
Using  \cref{estimativaorden3}, we have
\begin{equation}
\label{expansionauxiliar3}
\begin{split}
    \psi(v)-x&= \mathcal{P}_{{T_{x} \mathcal{M}}} (\psi(v)-x)+\mathcal{P}_{{T_{x} \mathcal{M}}^\perp} (\psi(v)-x)\\
    &=T(v)+\mathcal{P}_{{T_{x} \mathcal{M}}^\perp} (\psi(v)-x) + O (\|v\|^3).
    \end{split}
\end{equation}
Next, combining \cref{expansionwedgepruebafinal,expansionauxiliar1,expansionauxiliar2,expansionauxiliar3}, we obtain the following expression:

\begin{equation}
\label{ecuacionmasimportante}
\begin{split}
    \mathcal{P}&_{\bigwedge^{k+1} T_x \mathcal{M}}((\psi(v)-x)\wedge (w(\psi(v))-w(x)) q(\psi(v)))\\
    &=\sum_{I} \sum_j v_j (a_I(v)-a_I(0)) T(e_j) \wedge T(e_{i_1}) \wedge \cdots \wedge T(e_{i_k}) q(\psi(v)) + O( \| v \|^3)
    \end{split}
\end{equation}
Furthermore, since $\psi(v)-x=O(\|v\|^1)$ and $a_I$ is smooth, \cref{expansionwedgepruebafinal} shows that $w(\psi(v))-w(x)=O(\|v\|^1)$. Thus, we have

\begin{equation}
\label{ecuacionmasimportante2}
\mathcal{P}_{\bigwedge^{k+1} T_x \mathcal{M}}((\psi(v)-x)\wedge (w(\psi(v))-w(x)) q(\psi(v)))=O(\|v\|^2)
\end{equation}
We now apply the previous equations in conjunction with Lemmas \ref{lemaprinci}, \ref{lemma2}, and \ref{lema3desintegral} to complete the proof. Specifically, we use these lemmas for the following functions: 
\begin{equation} 
\label{kernelKmenor}
    K(x,y)= \mathcal{P}_{\bigwedge^{k+1}T_x \mathcal{M}}((y-x)\wedge (w(y)-w(x)) q(y)),
\end{equation}
 \begin{equation*}
 %\label{kernelSmenor}
  S(v)=\sum_{I} \sum_j v_j (a_I(v)-a_I(0)) T(e_j) \wedge T(e_{i_1}) \wedge \cdots \wedge T(e_{i_k}) q(\psi(v)),
 \end{equation*}
$Q(v)=v_i$ and $g(v)=(a_I(v)-a_I(0))q(\psi(v))$. Note that $g(0)=0$ and $\frac{\partial g}{\partial v_i}(0)=\frac{\partial a_I}{\partial v_i}(0) q(x)$. In this context, according to \cref{ecuacionmasimportante,ecuacionmasimportante2}, the parameters appearing in the hypotheses of the lemmas for $K(x,y)$ and $S(v)$ are $r=3$, $s=2$, and $l=1$. 

To prove the result, we first apply \cref{lema3desintegral} with the kernel $K$ defined in \cref{kernelKmenor}, which allows us to decompose the integral as follows: 

\begin{equation}
\label{expacioncasifinal}
   \mathcal{P}_{\bigwedge^{k+1} T_x \mathcal{M}} (\mathbf{P_t} w (x)) =\frac{1}{d_t(x)}( P_{t,\delta}(x) + O (t^{2+2(1-\delta)(d+2)}))
\end{equation}
for all $\frac{1}{2} < \delta <1$, where $\mathbf{P_t}w(x)$ is defined in \cref{operadorP}. Note that the kernel $K(x,y)$ defined above is not to be confused with $K_t(x,y)$ in \cref{operadorP}.

Next, applying Lemmas \ref{lemaprinci} and \ref{lemma2} to the previous functions $K$ and $S$, and using the rapid decay of the exponential function, we obtain the following expression: 
\begin{equation}
\label{integralfinalprueba}
\begin{split}
    P_{t,\delta}(x)=& \sum_{I} \sum_{j_1} \sum_{j_2} \paren{\int_{\R^d} v_{j_1} v_{j_2} e^{\frac{-\| T(v)\|^2}{2 t^2}} \frac{\partial a_{I}}{\partial v_{j_2}}(0) q(x) dv} (T(e_{j_1}) \wedge T(e_{i_1}) \wedge \cdots \wedge T(e_{i_k}) )\\
    & +O(t^{d+2+(4 \delta -2)}).
\end{split}
\end{equation}
Since the odd moments of the Gaussian are zero, the terms in \cref{integralfinalprueba} for which $j_1 \neq j_2$ are vanish. Therefore, \cref{integralfinalprueba} simplifies to:

\begin{equation}
\label{expacioncasifinal2}
\begin{split}
    P_{t,\delta}(x)=&(2 \pi)^{\frac{d}{2}} t^{d+2} q(x)\sum_{I} \sum_{j}  \frac{\partial a_{I}}{\partial v_{j}}(0) (T(e_{j}) \wedge T(e_{i_1}) \wedge \cdots \wedge T(e_{i_k}))\\
    & +O(t^{d+2+(4 \delta -2)})\\
    =&(2 \pi)^{\frac{d}{2}} t^{d+2} q(x) \mathbf{d}(w) (x) + O(t^{d+2+(4 \delta -2)}).
    \end{split}
\end{equation}
Now, combining \cref{expacioncasifinal}, \cref{expacioncasifinal2} and \cref{lemanor}, we obtain the following expression:

\begin{equation}\label{laecuacionfinalimportante}
       \mathcal{P}_{\bigwedge^{k+1}T_x \mathcal{M}} (\mathbf{P_t} w (x)) =\frac{(2 \pi)^{\frac{d}{2}} t^{d+2} q(x) \mathbf{d}(w) (x) + O(t^{d+2+(4 \delta -2)}) + O (t^{2+2(1-\delta)(d+2)})} {(2 \pi)^{\frac{d}{2}} t^d q(x)+O(t^{d+4 \delta -2})+O (t^{2(1-\delta)(d+2)})}.
\end{equation}
This estimate holds for all $\frac{1}{2}< \delta <1$. In particular, it holds for all $\delta$ satisfying the condition in \cref{deltacondicion}. For any such $\delta$, the exponents in \cref{laecuacionfinalimportante} satisfy $2(1-\delta)(d+2)>d$ and $0<4 \delta -2 <2$. Consequently, \cref{laecuacionfinalimportante} simplifies to: 

\begin{equation*}
       \mathcal{P}_{\bigwedge^{k+1}T_x \mathcal{M}} (\mathbf{P_t} w (x)) 
       =t^{2} (\mathbf{d}(w) (x) +O(t^{f})),
\end{equation*}
where the exponent $f$ is defined as $f=\min(4 \delta -2,2(1-\delta)(d+2))$.
$\hfill \blacksquare$\\

%\begin{thebibliography}{99}

%\end{thebibliography}

\end{document}